\documentclass[journal,twoside]{IEEEtran}
\usepackage{amsmath,amsfonts,amssymb}
\usepackage{algorithm}
\usepackage{algpseudocode}
\usepackage{array}
\usepackage[caption=false,font=normalsize,labelfont=sf,textfont=sf]{subfig}
\usepackage{textcomp}
\usepackage{stfloats}
\usepackage{url}
\usepackage{verbatim}
\usepackage{graphicx}
\usepackage{cite}
\newif\ifarxiv
\arxivtrue

\ifarxiv
    \newcommand{\includesvg}[2][]{\includegraphics[#1]{#2}}
\else
    \usepackage{svg}
    \svgsetup{
      inkscapepath=svg-inkscape,
      inkscapelatex=true
    }
\fi
\usepackage{siunitx}
\usepackage{tikz}
\usepackage{tikz-cd}
\usepackage{pgfplots}
\pgfplotsset{compat=1.18}
\usetikzlibrary{positioning, fit, backgrounds, calc, shapes.geometric, arrows, decorations.pathreplacing, bending, decorations.pathmorphing, fillbetween}
\usepackage{stmaryrd}

\usepackage{tabularx}
\usepackage{booktabs}
\usepackage{makecell}
\usepackage{multirow}
\usepackage{amsthm}

\usepackage[dvipsnames,svgnames,x11names]{xcolor}
\ifarxiv
    \usepackage{listings}
    \lstnewenvironment{minted}[2][]{%
        \lstset{
            basicstyle=\small\ttfamily,
            breaklines=true,
            breakatwhitespace=false,
            breakindent=0pt,
            columns=fullflexible,
            keepspaces=true,
            escapeinside={~}{~},
            xleftmargin=0pt,
        }%
    }{}
\else
    \usepackage{minted}
    \setminted{fontsize=\scriptsize, breaklines, breakanywhere, autogobble, breakautoindent=false}
\fi

\usepackage{hyperref}
\usepackage{cleveref}
\crefname{figure}{Fig.}{Figs.}
\Crefname{figure}{Fig.}{Figs.}
\crefname{equation}{}{}
\Crefname{equation}{}{}

\newif\ifcenterfigcaptions 
\makeatletter
\long\def\@makecaption#1#2{%
\ifx\@captype\@IEEEtablestring%
  \parbox[t]{\hsize}{\footnotesize\noindent #1.~~ #2}%
  \@IEEEtablecaptionsepspace%
\else%
  \@IEEEfigurecaptionsepspace%
  \setbox\@tempboxa\hbox{\footnotesize #1.~~ #2}%
  \ifdim \wd\@tempboxa >\hsize%
    \setbox\@tempboxa\hbox{\footnotesize #1.~~ }%
    \parbox[t]{\hsize}{\footnotesize \noindent\unhbox\@tempboxa#2}%
  \else%
    \ifcenterfigcaptions \hbox to\hsize{\footnotesize\hfil\box\@tempboxa\hfil}%
    \else \hbox to\hsize{\footnotesize\box\@tempboxa\hfil}%
  \fi\fi\fi}
\makeatother

\newcommand{\cmark}{\checkmark}
\newcommand{\xmark}{$\times$}
\newcommand{\notapp}{\textrm{---}}
\newcommand{\notenc}{{\footnotesize\textsc{n/e}}}

\newcommand{\satmark}{{\scriptsize SAT}}
\newcommand{\unsatmark}{{\scriptsize UNSAT}}
\newcommand{\optmark}{{\scriptsize OPT}}
\newcommand{\suboptmark}{{\scriptsize SUBOPT}}
\newcommand{\numpad}[1]{\makebox[1.0em][r]{#1}}

\newtheorem{proposition}{Proposition}

\begin{document}

\newif\ifsubmission
\submissionfalse

\title{Lifelong Robot Recomposition via Persistent Categorical Modeling for Unified Task-Driven Co-Design, Verification, and Planning}

\ifsubmission
    \author{Anonymous Authors}
\else
    \author{Steven Swanbeck and Mitch Pryor
    \thanks{The authors are with Texas Robotics and the Walker Department of Mechanical Engineering,
            The University of Texas at Austin, Austin, TX 78712, USA
            {\tt\small [stevenswanbeck, mpryor]@utexas.edu}}
    }
\fi

\maketitle

\newcommand{\sw}[1]{\texttt{#1}}
\newcommand{\ssw}[1]{\scriptsize\texttt{#1}}

\newcommand{\state}{s}
\newcommand{\States}{\mathcal{S}}
\newcommand{\action}{a}
\newcommand{\Actions}{\mathcal{A}}
\newcommand{\param}{\theta}
\newcommand{\Params}{\Theta}
\newcommand{\struct}{\mu}
\newcommand{\Structs}{\mathcal{U}}
\newcommand{\activation}{b}
\newcommand{\Activations}{\mathcal{B}}

\newcommand{\formula}{\Phi}
\newcommand{\sat}{\mathrm{SAT}}
\newcommand{\unsat}{\mathrm{UNSAT}}
\newcommand{\opt}{\mathrm{OPT}}
\newcommand{\subopt}{\mathrm{SUBOPT}}
\newcommand{\literal}{l}
\newcommand{\Literals}{L}
\newcommand{\soft}{\mathrm{soft}}
\newcommand{\hard}{\mathrm{hard}}

\newcommand{\conjunct}{\varphi}
\newcommand{\Conjuncts}{\mathcal{L}}
\newcommand{\mus}{\mathcal{M}}
\newcommand{\supp}{\mathrm{supp}}

\newcommand{\cat}{\mathcal{C}}
\newcommand{\obj}{\mathrm{Ob}(\cat)}
\newcommand{\morph}{\mathrm{Hom}(\cat)}
\newcommand{\id}{\mathrm{id}}
\newcommand{\resource}{r}
\newcommand{\Resources}{\mathcal{R}}
\newcommand{\view}{\mathrm{view}}

\newcommand{\pre}{\mathrm{Pre}}
\newcommand{\eff}{\mathrm{Eff}}
\newcommand{\post}{\mathrm{Post}}
\newcommand{\keep}{\mathrm{Keep}}

\newcommand{\init}{\mathrm{init}}
\newcommand{\goal}{\mathrm{goal}}
\newcommand{\req}{\mathrm{req}}

\newcommand{\vertsep}{0.35}
\newcommand{\bigvertsep}{0.5}

\tikzset{
    box/.style={
        draw,
        minimum height=0.5cm,
        minimum width=0.5cm,
        align=center,
        anchor=north,
        fill=blue!0,
        outer sep=0pt,
    },
    wire/.style={
        -,
        out=-90, in=90,
    },
    junction/.style={
        fill=black,
        circle,
        inner sep=1pt,
        outer sep=0pt,
        scale=1.2,
    },
    port predicate/.style n args={3}{
        draw,
        isosceles triangle,
        isosceles triangle apex angle=90,
        inner sep=0.75pt, outer sep=0pt, scale=1.2,
        shape border rotate=#1,
        anchor=#2,
        fill=#3,
    },
    expl/.style={
        font=\scriptsize,
        anchor=west,
        align=left,
    }
}

\newcommand{\AddPorts}[2]{%
  \pgfmathsetmacro{\N}{0}%
  \foreach \p/\nm/\lbl in {#2} {\pgfmathparse{\N+1}\global\let\N\pgfmathresult}%

  \foreach [count=\i] \p/\nm/\lbl in {#2} {%
      \pgfmathsetmacro{\fact}{\i/(\N+1)}%
      \def\boundcolor{black}%
      \def\freecolor{white}%

      \ifx\lbl\empty
        \def\placelabel{0}
      \else
        \def\placelabel{1}
      \fi

      \ifnum\pdfstrcmp{\p}{basic}=0
          \coordinate(\nm-in) at ($(#1.north west)!\fact!(#1.north east)$);
          \coordinate(\nm-out) at ($(#1.south west)!\fact!(#1.south east)$);
          \ifnum\placelabel=1
            \node[font=\tiny, anchor=south west] 
                at ($(\nm-in.north east)+(-0.25em,-0.25em)$) {$\lbl$};
          \fi
      \fi

      \ifnum\pdfstrcmp{\p}{read}=0
          \node[port predicate={90}{south}{\boundcolor}, name=\nm-in] 
              at ($(#1.north west)!\fact!(#1.north east)$) {};
          \coordinate(\nm-out) at ($(#1.south west)!\fact!(#1.south east)$);
          \ifnum\placelabel=1
            \node[font=\tiny, anchor=south west] 
                at ($(\nm-in.north east)+(-0.25em,-0.25em)$) {$\lbl$};
          \fi
      \fi

      \ifnum\pdfstrcmp{\p}{write}=0
          \coordinate(\nm-in) at ($(#1.north west)!\fact!(#1.north east)$);
          \node[port predicate={-90}{north}{\boundcolor}, name=\nm-eff] 
              at ($(#1.north west)!\fact!(#1.north east)$) {};
          \coordinate(\nm-out) at ($(#1.south west)!\fact!(#1.south east)$);
          \ifnum\placelabel=1
            \node[font=\tiny, anchor=south west] 
                at ($(\nm-in.north east)+(-0.25em,-0.25em)$) {$\lbl$};
          \fi
      \fi

      \ifnum\pdfstrcmp{\p}{readwrite}=0
          \node[port predicate={90}{south}{\boundcolor}, name=\nm-in] 
              at ($(#1.north west)!\fact!(#1.north east)$) {};
          \node[port predicate={-90}{north}{\boundcolor}, name=\nm-eff] 
              at ($(#1.north west)!\fact!(#1.north east)$) {};
          \coordinate(\nm-out) at ($(#1.south west)!\fact!(#1.south east)$);
          \ifnum\placelabel=1
            \node[font=\tiny, anchor=south west] 
                at ($(\nm-in.north east)+(-0.25em,-0.25em)$) {$\lbl$};
          \fi
      \fi

      \ifnum\pdfstrcmp{\p}{fread}=0
          \node[port predicate={90}{south}{\freecolor}, name=\nm-in] 
              at ($(#1.north west)!\fact!(#1.north east)$) {};
          \coordinate(\nm-out) at ($(#1.south west)!\fact!(#1.south east)$);
          \ifnum\placelabel=1
            \node[font=\tiny, anchor=south west] 
                at ($(\nm-in.north east)+(-0.25em,-0.25em)$) {$\lbl$};
          \fi
      \fi

      \ifnum\pdfstrcmp{\p}{fwrite}=0
          \coordinate(\nm-in) at ($(#1.north west)!\fact!(#1.north east)$);
          \node[port predicate={-90}{north}{\freecolor}, name=\nm-eff] 
              at ($(#1.north west)!\fact!(#1.north east)$) {};
          \coordinate(\nm-out) at ($(#1.south west)!\fact!(#1.south east)$);
          \ifnum\placelabel=1
            \node[font=\tiny, anchor=south west] 
                at ($(\nm-in.north east)+(-0.25em,-0.25em)$) {$\lbl$};
          \fi
      \fi

      \ifnum\pdfstrcmp{\p}{freadfwrite}=0
          \node[port predicate={90}{south}{\freecolor}, name=\nm-in] 
              at ($(#1.north west)!\fact!(#1.north east)$) {};
          \node[port predicate={-90}{north}{\freecolor}, name=\nm-eff] 
              at ($(#1.north west)!\fact!(#1.north east)$) {};
          \coordinate(\nm-out) at ($(#1.south west)!\fact!(#1.south east)$);
          \ifnum\placelabel=1
            \node[font=\tiny, anchor=south west] 
                at ($(\nm-in.north east)+(-0.25em,-0.25em)$) {$\lbl$};
          \fi
      \fi

      \ifnum\pdfstrcmp{\p}{freadwrite}=0
          \node[port predicate={90}{south}{\freecolor}, name=\nm-in] 
              at ($(#1.north west)!\fact!(#1.north east)$) {};
          \node[port predicate={-90}{north}{\boundcolor}, name=\nm-eff] 
              at ($(#1.north west)!\fact!(#1.north east)$) {};
          \coordinate(\nm-out) at ($(#1.south west)!\fact!(#1.south east)$);
          \ifnum\placelabel=1
            \node[font=\tiny, anchor=south west] 
                at ($(\nm-in.north east)+(-0.25em,-0.25em)$) {$\lbl$};
          \fi
      \fi

      \ifnum\pdfstrcmp{\p}{readfwrite}=0
          \node[port predicate={90}{south}{\boundcolor}, name=\nm-in] 
              at ($(#1.north west)!\fact!(#1.north east)$) {};
          \node[port predicate={-90}{north}{\freecolor}, name=\nm-eff] 
              at ($(#1.north west)!\fact!(#1.north east)$) {};
          \coordinate(\nm-out) at ($(#1.south west)!\fact!(#1.south east)$);
          \ifnum\placelabel=1
            \node[font=\tiny, anchor=south west] 
                at ($(\nm-in.north east)+(-0.25em,-0.25em)$) {$\lbl$};
          \fi
      \fi
  }
}

\begin{abstract}
Robotic systems are traditionally designed and deployed in static configurations, with assumptions made at design-time becoming immutable constraints during runtime.
This design-then-deploy paradigm produces performant systems under narrow operating conditions, but renders robots brittle when qualities of themselves, their tasks, or their environments unexpectedly change.
We address this challenge with a compositional framework that formalizes robotic systems as abstract circuits within a strict symmetric monoidal category, in which design and runtime composition of hardware, software, and behavior are synthesized simultaneously via an SMT-based solver, with monoidal functors projecting the system into lifecycle-specific views and free symbolic variables simultaneously solving for parameters and entire component specifications within larger compositions.
This persistent model also supports queries a long-lived system needs beyond plan existence across its entire lifecycle, including mapping Pareto fronts over candidate compositions, diagnosing why a composition has become infeasible, finding its minimal restoration, and reconfiguring with limited change to the deployed system.
We evaluate against official implementations of optimal numeric, stream-based, and SMT-based planners all measured onboard a deployed robot and demonstrate the approach end-to-end in a search-and-rescue scenario in which the robot recognizes when it has become unfit and synthesizes and assumes new holistic configurations to restore operation.
We release our solver and supporting software open-source\footnote{\label{fn:code}Link will be provided in the final manuscript.}.
\end{abstract}
\ifsubmission
    \begingroup
\renewcommand{\abstractname}{Note to Practitioners}
\begin{abstract}
Autonomous systems are usually built to perform fixed jobs in controlled environments, and require significant human intervention when the mission, environment, or system itself changes unexpectedly.
This work was motivated by the challenge of deploying autonomy into dynamic environments where the tasks that must be performed and therefore the capabilities needed may not be known in advance.
We demonstrate that rather than partitioning the system's lifecycle into distinct stages of design, where hardware and software are typically integrated, and deployment, where most systems only leave behavior mutable, we can instead use a unified approach to holistically reason about the system's hardware, software, and behavior across lifecycle stages.
This approach is most useful in applications of autonomy where there is irreducible uncertainty, limited prior knowledge, and a need for rapid response, such as field operations, search and rescue, and exploration in unknown environments.
Adoption of this approach requires that every usable component be faithfully described once by its requirements and effects since the system can only reason over what is modeled, and integration with suitable abstractions like those we reference to ground the high-level holistic reasoning to the low-level engineering required to realize useful real-world systems.
We frame and demonstrate this work within robotics because it represents an extreme case of cross-domain hardware, software, and behavior interactions.
But because our approach abstracts away the specifics of these individual domains, it generalizes to automated systems more broadly.
\end{abstract}
\endgroup

\fi
\section{Introduction}
\label{sec:introduction}

\ifsubmission
    \begin{figure}[t]
    \centering
    \includegraphics[width=0.47\textwidth]{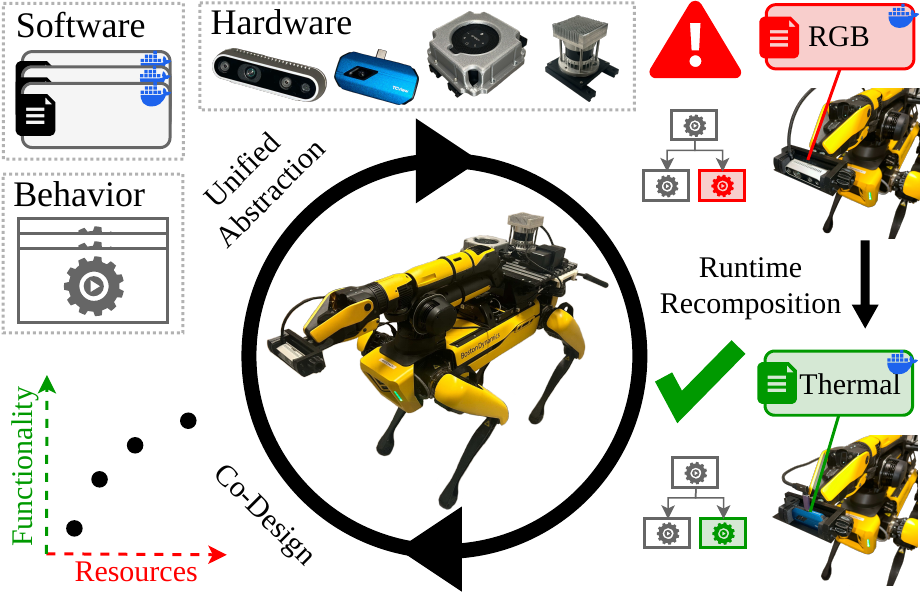}
    \caption{Informal graphical illustration of lifelong robot recomposition. Via unified compositional abstraction, hardware, software, and behavior modules can be consistently reasoned over and composed. This process extends beyond pre-deployment co-design to also support recomposition of the system during runtime in response to unexpected changes in system capabilities, mission requirements, or the environment.}
    \vspace{-1ex}
    \label{fig:cover}
\end{figure}
\fi

\IEEEPARstart{R}{obotic} system integration is a notoriously difficult problem.
Modern robots are composed of heterogeneous mechanical, electrical, and software subsystems that must be carefully orchestrated to achieve high-level goals \cite{citizen_developer_framework_2022}.
Due to this complexity, ad hoc practices dominate the integration process, leading to monolithic designs that are brittle and resistant to change \cite{software_variability_2023}.
Traditionally, the lifecycle of robotic systems is partitioned into discrete, disconnected phases, which are broadly categorized into \emph{design} and \emph{deploy} stages. 
For a system integrated using this conventional approach, its physical morphology and software stack are fixed during design time, and, once deployed, it is expected to remain in this fixed state for its entire operational lifespan with the exception of occasional manual software updates and online high-level task-planning \cite{component_based_robot_engineering_p1_2009,software_reconfiguration_2025}.
Even these routine tasks, let alone substantial upgrades, are completed by robotic experts in the lab.

While this design-then-deploy paradigm produces reliable and performant systems for narrow applications, it imposes critical limitations.
First, a robot with a static configuration may not be able to gracefully adapt when its mission changes or its fitness for that mission deteriorates based on qualities of itself or its environment, such as an unexpected failure in a mission-critical sensor or software module or an unanticipated requirement to navigate an inaccessible type of terrain \cite{robots_adapt_like_animals_2015,long_term_ai_survey_2018}.
Second, because the design of the robot's hardware, software, and eventual task execution are performed by different actors, the resulting system is often a suboptimal compromise rather than a task-aligned optimal solution \cite{co_design_beyond_monotone_2019,cosmic_2026}.
Finally, the integration process is fragmented, both between domains and even within the same domain.
Integration of a new sensor to a robot poses separate challenges from the integration of a new software module, for example. 
And integrating that sensor to an open-world robot is a significantly different process than doing so on a robot designed for a factory setting.
To combat these challenges, \emph{abstractions} such as standardized mechanical and electrical interfaces, middleware, and component-based software architectures already support composition in some domains \cite{component_based_robot_engineering_p1_2009,software_variability_2023,ros2_2022}.
While such abstractions exist to support many of the subsystems that contribute to a robot, the lack of a higher-level abstraction to reason across subsystems means top-level system integration remains a fragmented, ad hoc process for every new configuration.

\unless\ifsubmission
    \begin{figure}[t]
    \centering
    \includegraphics[width=0.47\textwidth]{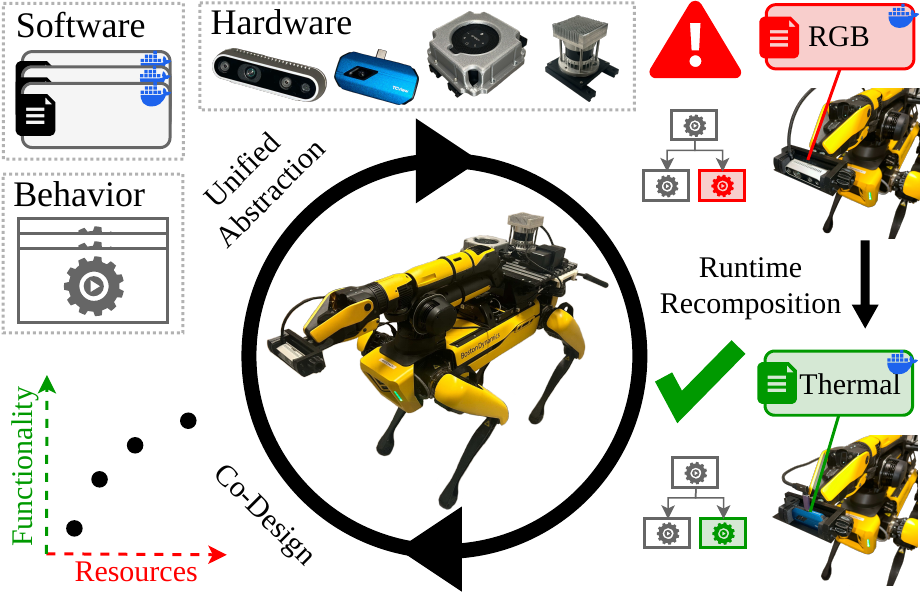}
    \caption{Informal graphical illustration of lifelong robot recomposition. Via unified compositional abstraction, hardware, software, and behavior modules can be consistently reasoned over and composed. This process extends beyond pre-deployment co-design to also support recomposition of the system during runtime in response to unexpected changes in system capabilities, mission requirements, or the environment.}
    \vspace{-1ex}
    \label{fig:cover}
\end{figure}
\fi


Closely tied to the idea of abstraction is \emph{compositionality}, which describes the ability to recursively reason about a larger system's behavior based on the behaviors and interactions of its constituent parts \cite{compositional_thinking_2022}.
The recursive property of compositional reasoning is made possible by treating the constituent parts as \emph{black boxes} with well-defined \emph{interfaces}, which allows for the construction of larger systems without requiring knowledge of the internal workings of each part.
This is the abstraction underlying extensive recent work in co-design \cite{co_design_theory_2015,compositional_thinking_2022,codei_2025}, which aims to jointly optimize subsystems including hardware morphology, sensing, and control to produce more performant and effective task-driven autonomy.
In co-design literature, the interfaces of each part are expressed as the resources they require and the functionalities they provide, which can be supplied or consumed by other parts in the system \cite{co_design_theory_2015}.

Abstracted components with standardized interfaces enable the cross-domain reasoning required in co-design, but also have broader implications for the accessibility of robotic system integration.
Recent work in online system reconfiguration has demonstrated how systems can reliably integrate new capabilities discovered during runtime with no prior knowledge via compositional abstractions \cite{no_reboot_required_2025,deployment_is_not_destiny_2026}.
While the reconfiguration in these works enables a system to integrate new capabilities, the integration process, including deciding which modules to integrate to the system at which times, is still performed by a human operator, with only high-level task-planning being performed autonomously.
So, although abstractions enable new capabilities to be integrated reliably, they have not yet been generalized to enable the high-level cross-domain reasoning required for autonomous online system integration that subsumes design-time co-design and runtime task-planning.

Building on these existing efforts, we propose a new approach we term \emph{recompositional} robotics, where a system is composed of heterogeneous hardware, software, and behavior modules that can be dynamically reconfigured throughout the system's operational lifespan.
In this view, the distinct problems of design-time co-design and runtime task-planning are unified into a single, continuous optimization problem, with all capabilities that are typically captured in the design stage or the deploy stage becoming mutable throughout the system's lifespan.
In addition to solving for optimal system compositions, there are several other questions a long-lived system must answer, including why a configuration has become infeasible, what its minimal restoration is, and to what extent the deployed system must change to restore operation.
We demonstrate that these questions can all be answered as a family of queries over a persistent categorical model of the system that is continuously updated and reevaluated as the system evolves, and argue that the boundary drawn between design-time and runtime integration is an artifact of engineering practice rather than a fundamental limitation of the problem itself.

The main contributions of this paper are as follows:
\begin{itemize}
    \item A category-theoretic \cite{category_theory_1945} modeling framework in which hardware, software, and behavior components are all morphisms in a free strict symmetric monoidal category and system compositions can be formed and interrogated as categorical circuits.
    \item A translation of this categorical model into a satisfiability modulo theories (SMT) formulation that enables real-time synthesis of optimal compositions that can also include free parametric and structural variables representing unknown values and entire missing components, simultaneously solved in a single process.
    \item An extension of this formulation to support lifelong queries over the same model including diagnosing infeasibility, generating minimal relaxations, and reconfiguring with minimal change to the deployed system.
    \item Applications spanning the lifecycle of a real robotic system, from design-time Pareto-optimal component design to runtime recomposition during a search-and-rescue deployment in which the robot synthesizes and assumes new configurations online in response to changing conditions.
    \item An evaluation of the above against optimal numeric, stream-based, and SMT-based planners, measured onboard a deployed robot. 
\end{itemize}

\section{Related Works}
\label{sec:related_works}

The challenges of integrating hardware design, software composition, and task planning have traditionally been addressed in a compartmentalized manner.
While recent research efforts have moved toward unified frameworks for simultaneous optimization of diverse subsystems, most approaches remain confined to specific stages of the robot's lifecycle.

\subsection{Modular, Reconfigurable, and Self-Reconfigurable Systems}
A large body of work studies robots whose \emph{morphologies} can change.
These robots assemble a target structure from joinable hardware modules, with self-reconfigurable systems coordinating their own docking, actuation, and motion planning to assume new configurations \cite{trends_in_reconfigurable_modular_robots_2017,modular_reconfigurable_review_2019,modular_reconfigurable_review_2025}.
Many modular systems can recognize when their topologies change and regenerate kino-dynamic and geometric models automatically when modules are mated \cite{concert_2026}, and some approaches can even be provided a task and asked which feasible morphology best suits it \cite{design_and_pose_optimization_2024}.
Across this literature, it is assumed that reconfiguration \emph{is} morphological, where the degrees of freedom considered in the reconfiguration process correspond to physical connections between modules and their relative positions and orientations, and adaptation is achieved by rearranging known modules into a new structure.
These approaches can propagate the updated structure to downstream motion planning and control, but more general \emph{capabilities} afforded by a new configuration or interactions with other subsystems beyond kino-dynamics are not captured.
Broader notions of adaptive robotics still frame adaption primarily around control and behavior rather than the holistic system \cite{adaptive_robotics_methodology_2022}, while other work investigates behavioral response to physical damage \cite{robots_adapt_like_animals_2015}, drawing inspiration from how animals can recognize injuries and adjust their actions accordingly.
Extensive work in Plug-and-Produce methodologies within the manufacturing space addresses a similar need by treating reconfigurability as a property designed into workcells so that operational flexibility and efficiency are preserved as requirements and technologies evolve \cite{plug_and_produce_review_2024}.

In our setting, the components are heterogeneous and span layers that include electromechanical hardware modules, software stacks, and behavior-based task execution.
The result is that physical morphology is only one axis along which a system can be reconfigured, and it is often not the cheapest option.
For example, swapping out a sensor driver or re-parameterizing a behavior tree may restore operation and likely can be performed more quickly and autonomously than exchanging a physical sensor.
Accordingly, the configurations generated by our approach are not physical structures, but rather heterogeneous compositions of black-boxed modules defined only by their interface obligations.
Morphology can be captured within these compositions, but it is treated as one class of variable among many.
Because the term ``modular'' is so strongly tied to morphological reconfiguration and often also implies self-reconfiguration from homogeneous units in the existing literature, we describe this broader class as \emph{recompositional}.
This notion aligns naturally with the compositional, categorical approach motivated in the remainder of this section and described in \Cref{sec:approach}, where every system modification, physical or otherwise, is a morphism in one shared category.

\subsection{Self-Adaptation and Models at Runtime}
Automatic and self-adaptive computing has established the pattern of a system maintaining a runtime model of itself and closing an iterative monitor-analyze-plan-execute cycle over it \cite{automatic_computing_vision_2003,self_managing_systems_2007}.
The models@run.time approach treats the model itself as the adaptation artifact \cite{models_at_runtime_2009,models_at_runtime_2019}, and similar work in robotics reconfigures ROS-based architectures when components degrade \cite{mros_2022}.
More broadly, software reconfiguration at runtime has been explored extensively in robotics, though techniques beyond parameter reconfiguration remain rare in practice \cite{software_reconfiguration_2025}.
We share the premise of a live self-model, but differ in what that model can generate.
The adaptation space in these frameworks is a space of software architecture variants enumerated at design time.
Ours synthesizes configurations spanning hardware, software, and behavior that are not drawn from a predefined or enumerated set of candidates, can include values outside a modeled or enumerated vocabulary, and exposes the model as a logical artifact that can answer queries about why it is infeasible and how much it needs to change rather than only providing an updated configuration.

\subsection{Categorical and Compositional Modeling}
One of the core challenges in robotics and embodied autonomy is describing the complex interactions between heterogeneous components that impose vastly different constraints.
Category theory \cite{category_theory_1945} has emerged as a prominent tool for formalizing similar interactions in engineering and sciences through compositionality \cite{category_theory_for_sciences_2014,seven_sketches_in_compositionality_2018,compositional_thinking_2022}.
Specifically, resource theories provide a way to reason over how different inputs are transformed into outputs \cite{resource_theory_2016} and work on the mathematical frameworks for composing wiring diagrams has been demonstrated in applications such as plug-and-play circuits and dynamic systems analysis \cite{operad_wiring_diagrams_2013,dynamic_systems_operad_wiring_diagrams_2014}.
Previous robot-aligned categorical solutions include efforts in co-design \cite{co_design_theory_2015}, which are discussed further in the following subsection, and frameworks for translating high-level specifications into action that is interoperable between systems and contexts \cite{robocat_2021,task_plan_transfer_2025}.
Although these efforts formalize compositional reasoning and interoperability, leveraging such structure to enable autonomous re-synthesis of embodied configurations under changing constraints remains an unexplored problem.

\subsection{Task-Driven and Computational Co-Design}
The joint optimization of hardware and software has been rigorously addressed through Censi's mathematical theory of co-design \cite{co_design_theory_2015}.
Based on monotone co-design problems (MCDPs), this approach models design as a monotone feasibility relation between functionality and resources, solving for Pareto-optimal solutions that maximize functionality for given resources or minimize resources to achieve a desired functionality.
This work has been expanded in many subsequent publications, including optimizing drone hardware and control policies \cite{co_design_hardware_control_2021} and self-driving vehicle stacks \cite{co_design_embodied_intelligence_2021,codei_2025}.
The primary artifact these methods report is the Pareto front over design candidates.
In our approach, Pareto fronts are generated in one query (\Cref{subsec:applications:design}) over the same model that also performs synthesis and runtime adaptation, rather than requiring a dedicated design-time tool.

A related thread of work is computational co-design, which explores similar design questions \cite{robotic_lifeforms_2000} and has been demonstrated for structural design and simultaneous design of morphology and control over continuous or grammar-generated spaces \cite{robogami_2017,robogrammar_2020}.
This thread makes similar assumptions about the scope of configuration as the modular robotics literature, searching a space generated by a fixed vocabulary of physical primitives and fixing the granularity at which the system is described to typically that of individual physical modules.
Searching within that scope is powerful when the objective is a novel morphology, but it limits the result to a detailed physical model with no natural expression of software or behavior.
We instead reason over black-boxed components described only by the interfaces they require and provide, which is what admits heterogeneous components into a single search.

Both co-design threads target design-time, solving for an optimal configuration within a fixed operational envelope, but not providing a mechanism for the robot to autonomously re-synthesize its configuration during runtime when internal or external constraints change.

\newcommand{\ymark}{$\bullet\!\!\!\circ$}
\newcommand{\pmark}{$\circ$}
\newcommand{\nmark}{$\times$}
\newcommand{\namark}{--}

\begin{table*}[t]
\centering
\caption{Positioning of this work against the threads of Section~\ref{sec:related_works}.
\emph{Components} is what the method has agency over; 
\emph{lifecycle} describes the stage during which each tool is used;
\emph{configuration space} is the set from which a configuration is drawn;
\emph{queries answered} are the questions the method can pose over its own model.
\ymark~= supported, \pmark~= supported in a restricted form (see notes), \nmark~= not supported, \namark~= not applicable.
The separation is structural, as the component and query blocks are satisfied by disjoint sets of rows.
Diagnosis, relaxation, and stability are mature capabilities over plans and planning domains, not over deployed systems, while methods that do reconfigure heterogeneous systems support no queries beyond the configurations they produce.}
\label{tab:positioning}
\footnotesize
\setlength{\tabcolsep}{4pt}
\begin{tabular*}{\textwidth}{@{\extracolsep{\fill}}l c c c l l c c c c c c@{}}
\toprule
& \multicolumn{3}{c}{\textbf{Components}}
& \multirow{2}{*}{\textbf{Lifecycle}}
& \multirow{2}{*}{\makecell[c]{\textbf{Configuration}\\\textbf{Space}}}
& \multicolumn{6}{c}{\textbf{Queries Answered}} \\
\cmidrule(lr){2-4} \cmidrule(lr){7-12}
\textbf{Thread \& Representative Work}
    & HW 
    & SW 
    & Behavior 
    & 
    & 
    & Exist.\ 
    & Opt.\ 
    & Diag.\ 
    & Relax 
    & Stab.\ 
    & Pareto \\
\midrule

Modular/self-reconfigurable robots \cite{modular_reconfigurable_review_2019,modular_reconfigurable_review_2025}
    & \ymark 
    & \nmark 
    & \nmark 
    & Runtime 
    & Unit lattice
    & \ymark 
    & \pmark\rlap{$^{a}$} 
    & \nmark 
    & \nmark 
    & \pmark\rlap{$^{b}$} 
    & \nmark \\

Plug-and-produce/reconfigurable manufacturing \cite{plug_and_produce_review_2024}
    & \ymark 
    & \pmark 
    & \nmark 
    & Both 
    & Device catalog
    & \ymark 
    & \pmark 
    & \nmark 
    & \nmark 
    & \nmark 
    & \nmark \\

Self-adaptation, models@run.time \cite{self_managing_systems_2007,models_at_runtime_2009,mros_2022}
    & \nmark 
    & \ymark 
    & \pmark 
    & Runtime 
    & Enum.\ variants
    & \ymark 
    & \pmark\rlap{$^{c}$} 
    & \pmark\rlap{$^{d}$} 
    & \nmark 
    & \pmark\rlap{$^{e}$} 
    & \nmark \\

Categorical/compositional modeling \cite{seven_sketches_in_compositionality_2018,compositional_thinking_2022,robocat_2021}
    & \pmark 
    & \pmark 
    & \pmark 
    & Both 
    & \namark 
    & \namark 
    & \namark 
    & \namark 
    & \namark 
    & \namark 
    & \namark \\

Task-driven co-design (MCDP) \cite{co_design_theory_2015,codei_2025}
    & \ymark 
    & \ymark 
    & \pmark 
    & Design 
    & Design catalog
    & \ymark 
    & \ymark 
    & \nmark 
    & \nmark 
    & \nmark 
    & \ymark \\

Computational co-design \cite{robogami_2017,robogrammar_2020}
    & \ymark 
    & \pmark\rlap{$^{f}$} 
    & \nmark 
    & Design 
    & Grammar / cont.
    & \ymark 
    & \ymark 
    & \nmark 
    & \nmark 
    & \nmark 
    & \pmark \\

Formal verification \cite{computer_aided_modular_verification_2017,compositional_modular_verification_2022}
    & \nmark 
    & \nmark 
    & \nmark 
    & Both 
    & Given
    & \nmark 
    & \nmark 
    & \pmark\rlap{$^{g}$} 
    & \nmark 
    & \nmark 
    & \nmark \\

Task planning, PDDL/TAMP \cite{pddl_1998_article}
    & \nmark 
    & \nmark 
    & \ymark 
    & Runtime 
    & Enum.\ / disc.
    & \ymark 
    & \ymark 
    & \nmark 
    & \nmark 
    & \nmark 
    & \nmark \\

Task planning, SMT backend \cite{pddl_smt_2020,smtplanplus_2016}
    & \nmark 
    & \nmark 
    & \ymark 
    & Runtime 
    & Exact numeric
    & \ymark 
    & \nmark\rlap{$^{h}$} 
    & \nmark 
    & \nmark 
    & \nmark 
    & \nmark \\

Diagnosis, repair, plan stability \cite{minimum_unsatisfiable_subsets_2008,planning_excuses_2010,plan_stability_2006}
    & \nmark 
    & \nmark 
    & \pmark\rlap{$^{i}$} 
    & Runtime 
    & Given plan
    & \namark 
    & \namark 
    & \ymark 
    & \ymark 
    & \ymark 
    & \nmark \\

\midrule
\textbf{Ours}
    & \ymark
    & \ymark 
    & \ymark 
    & \textbf{Both} 
    & \textbf{Synthesized}
    & \ymark 
    & \ymark 
    & \ymark 
    & \ymark 
    & \ymark 
    & \ymark \\
\bottomrule
\end{tabular*}

\vspace{2pt}
\scriptsize
$^{a}$Minimizes reconfiguration moves or actuation cost, not a task-level objective.\quad
$^{b}$Minimal-move self-reconfiguration captures churn but is confined to morphology.\quad
$^{c}$Utility-based selection over an enumerated variant set.\quad
$^{d}$Functional-level degradation diagnosis; no minimal set of constraints that failed.\quad
$^{e}$Adaptation cost and hysteresis are weighed, but the deployed configuration is not a first-class constraint.\quad
$^{f}$Control policy is co-optimized; the software stack itself is not a design variable.\quad
$^{g}$Counterexample traces explain violations of a checked property, not infeasibility of a composition.\quad
$^{h}$Satisficing: the SMT backend binds continuous parameters exactly but no objective is optimized.\quad
$^{i}$Operates on the plan or the planning domain, never on the deployed system.\quad
\end{table*}

\subsection{Formal Verification}
Verification focuses on the correctness of system compositions, which is especially critical in reconfigurable robots where a specific configuration may be untested prior to deployment \cite{computer_aided_modular_verification_2017,compositional_modular_verification_2022}.
Outside of robotics, solution approaches such as mapping to satisfiability problems \cite{practical_applications_boolean_satisifiability_2008} have become longstanding successful approaches for reasoning over systems with a large parameter space, and extensions beyond the Boolean satisfiability (SAT) problem to satisfiability modulo theory (SMT) problems have enabled reasoning over more complex formulas that can include arithmetic datatypes, strings, quantifiers, and more \cite{smt_introduction_2011}.
Previous works have explored verification of factors including gravitational stability, self-collision, and respecting actuator force limits in modular systems \cite{computer_aided_modular_verification_2017}, while other work has emphasized verification of robotic software architectures based on ROS using timed-automata \cite{ros_based_formal_verification_2017} and assume-guarantee reasoning \cite{compositional_modular_verification_2022}.
However, while many works have identified the need for verification in compositional systems, they have not captured the full scope of system-level configuration, which requires simultaneous reasoning over the complex interactions between hardware and software subsystems, actions performed by the resulting system, and the environment, meaning they do not generalize to our recompositional robotics setting.

\subsection{Task Planning}
Task planning focuses on the scheduling and utilization of available resources during deployment to achieve a provided task.
Traditional approaches, such as PDDL-based planning \cite{pddl_1998_article}, treat the system as a black box with fixed capabilities captured via a set of high-level actions.
Prior works have introduced category theory to planning, including categorical expression of PDDL plans, interoperability in robot programming, and translation of task plans between domains \cite{pddl_categry_theory_2021,robocat_2021,task_plan_transfer_2025}.
Other works have mapped the planning problem to SAT or SMT instances \cite{planning_as_satisifiability_2006,pddl_smt_2020}, with \cite{pddl_smt_2020} specifically mapping PDDL instances into SMT, and SMTPlan+ \cite{smtplanplus_2016} compiling PDDL+ hybrid domains to SMT for iterative-deepening satisficing search.
While these works help solve the task-planning problem, they struggle to express problems in runtime coordination between components and system design prior to deployment without significant manual alignment to a specific problem.
Frequent manual re-expression limits autonomy and responsiveness in the recompositional robotics setting, where a long-lived system must answer various queries about its state and capabilities that change over time.  
In our approach, we treat the task planning problem as part of the larger recomposition problem, expanding its scope beyond composing high-level actions to jointly consider the hardware and software components that enable or contribute to the capabilities of the actions.
We use SMT as our planning backend, and evaluate the ability of our approach to express problems relevant to the recompositional robotics setting against other planning mechanisms, including those backed by SMT, in \Cref{sec:evaluation}.

\subsection{Diagnosis, Repair, and Plan Stability}
Explaining infeasibility via minimal unsatisfiable subsets and restoring it via minimal correction sets is well-studied in constraint satisfaction and SAT/SMT solving \cite{minimum_unsatisfiable_subsets_2008,mus_finding_2013}.
In planning, related threads generate excuses for unsolvable tasks \cite{planning_excuses_2010} or repair the domain model directly when no plan can be generated \cite{model_repair_in_planning_survey_2025}, while work on plan stability generates successor plans that minimally perturb the current plan when it becomes invalidated \cite{plan_stability_2006}.
The threads operate in isolation on the plan or in the planning domain, but each of these capabilities is useful in a long-lived, persistent system that is expected to adapt to change during operation.
Because our formulation generates the joint hardware--software--behavior composition rather than a standalone task plan, the same machinery acquires broader system-level meaning, allowing the robot to explain its infeasibility, generate a relaxation for a concrete resource that would restore feasibility, and prioritize stability to minimize change to the incumbent composition rather than re-solving from scratch.
These capabilities are discussed in \Cref{subsec:approach:queries}.

\subsection{Positioning}
In the context of the recompositional robotics problem, the primary limitation across these threads is that the stages remain decoupled: co-design optimizes the robot within a fixed operational envelope, planning then treats the resulting capabilities as a fixed constraint, and verification acts as an external check on an existing configuration that can signal when it is violated, but cannot directly guide the re-synthesis because it uses independent machinery.
That is, because the stages are mathematically and operationally separate, a planning failure cannot trigger a hardware or software reconfiguration, and the system is brittle to any violation of its design-time assumptions.
A framework that unifies these domains and approaches them using the same formulation can instead resolve failures by holistically re-synthesizing its own structure and capabilities, remaining viable under shifting environments, internal changes such as hardware failures, or entirely new task objectives.
Unification also enables the diagnosis, relaxation, stability, and design-space queries that the above threads develop separately, each against its own restricted model, as a single set of queries over one shared model of the whole system.
Our positioning to prior work is summarized in \Cref{tab:positioning}.
\section{Preliminaries}
\label{sec:preliminaries}

This section establishes the mathematical foundations of our framework prior to \Cref{sec:approach}, which presents our approach. 
We first introduce the categorical structures later used to model robot compositions and then define the state transition system used for planning.
For a deeper discussion of these topics, we refer the reader to \cite{seven_sketches_in_compositionality_2018,compositional_thinking_2022,automated_planning_2004}.

\subsection{Monoidal Category Theory}
\label{subsec:prelims:category_theory}
A \emph{category} $\cat$ is a collection of \emph{objects} $\obj$ and \emph{morphisms} $\morph$ that map between objects.
A morphism $f: A \to B$ represents a transformation from an input type $A$ to an output type $B$.
The following must hold in $\cat$:
\begin{itemize}
    \item \textit{Identity:} For every $A \in \obj$, there is a morphism $\id_A: A \to A$ that preserves the object.
    \item \textit{Sequential Composition:} The binary composition operation $\circ$ is applied to morphisms when the codomain of the first morphism equals the domain of the second:
    \begin{equation}
        \begin{array}{c}
            f: A \to B, \quad g: B \to C \\
            g \circ f: A \to C.
        \end{array}
    \end{equation}
    The composition is \emph{associative}
    \begin{equation}
        \begin{array}{c}
            h: C \to D \\
            h \circ (g \circ f) = (h \circ g) \circ f,
        \end{array}
    \end{equation}
    and respects \emph{identity}
    \begin{equation}
        \id_B \circ f = f = f \circ \id_A.
    \end{equation}
\end{itemize}

These properties of $\cat$ are illustrated graphically in the following diagrams:
\[
\begin{array}{c@{\quad\quad\quad\quad}c@{\;=\quad}c}
    \vcenter{\hbox{\begin{tikzpicture}
    \node(A){$A$};
    \node[below=0.5cm of A](B){$B$};
    \node[below=0.5cm of B](C){$C$};
    \node[above=0.2cm of A](IDA){};
    \draw[->] (A) to node[midway, left] {$f$} (B);
    \draw[->] (B) to node[midway, left] {$g$} (C);
    \draw[->] (A) to [out=-45, in=45] node[midway, right] {$g \circ f$} (C);
    \draw[->](A.east) arc(-45:225:0.35) node[midway, right, xshift=3mm] {$\id_A$};
\end{tikzpicture}}}
&
    \vcenter{\hbox{\begin{tikzpicture}
    \coordinate (A) at (0,0);
    \node[box](f) at ($ (A) - (0,\bigvertsep) $) {$f$};
    \AddPorts{f}{basic/f1/};
    \node[box](g) at ($ (f.south) - (0,\bigvertsep) $) {$g$};
    \AddPorts{g}{basic/g1/};
    \coordinate (C) at ($ (g.south) - (0,\bigvertsep) $);

    \draw[] (A) -- node[midway, right, xshift=1mm] {$A$} (f1-in);
    \draw[] (f1-out) -- node[midway, right, xshift=1mm] {$B$} (g1-in);
    \draw[] (g1-out) -- node[midway, right, xshift=1mm] {$C$} (C);
\end{tikzpicture}}}
&
    \vcenter{\hbox{\begin{tikzpicture}
    \coordinate (A) at (0,0);
    \node[box](gf) at ($ (A) - (0,\bigvertsep) $) {$g \circ f$};
    \AddPorts{gf}{basic/p/};
    \coordinate (C) at ($ (gf.south) - (0,\bigvertsep) $);
    
    \draw[] (A) -- node[midway, right, xshift=1mm] {$A$} (p-in);
    \draw[] (p-out) -- node[midway, right, xshift=1mm] {$C$} (C);
\end{tikzpicture}
}}
\end{array}
\]
On the left, a \emph{commutative} diagram expresses these algebraic relationships.
The right shows these same relationships in a \emph{string diagram}, which depicts morphisms as boxes and objects as wires.
We prefer string diagrams throughout this work as they provide an intuitive topological representation of system connectivity, where wires correspond to physical or informational resource flows and boxes represent the processes that transform them.

A \emph{monoidal category} extends this basic structure with parallel composition via the tensor product bifunctor $\otimes: \cat \times \cat \to \cat$ and a unit object $I$.
The tensor product allows for grouping of objects $(A \otimes B)$ and the simultaneous application of morphisms:
\begin{equation}
    f \otimes h: A \otimes C \to B \otimes D.
\end{equation}
We assume $\cat$ is \emph{strict}, meaning the structural isomorphisms for associativity and units are treated as equalities: 
\begin{equation}
\label{eq:ct:strict_properties}
    \begin{array}{c}
         A \otimes (B \otimes C) = (A \otimes B) \otimes C \\
         I \otimes A = A = A \otimes I.
    \end{array}
\end{equation}
These monoidal properties are summarized in the following string diagrams:
\[
\begin{array}{c@{\quad\quad\quad\quad}c}
    \vcenter{\hbox{\begin{tikzpicture}
    \coordinate (A) at (-0.5,0);
    \coordinate (C) at (0.5,0);

    \coordinate (mid) at ($ (A)!0.5!(C) $);

    \node[box](fh) at ($ (mid) - (0,\bigvertsep) $) {$f \otimes h$};
    \AddPorts{fh}{basic/f/, basic/h/};

    \coordinate (lower) at ($ (fh.south) - (0,\bigvertsep) $);
    \coordinate (B) at (A |- lower);
    \coordinate (D) at (C |- lower);

    \draw[wire] (A) to node[midway, left, xshift=-1mm] {$A$} (f-in);
    \draw[wire] (C) to node[midway, right, xshift=1mm] {$C$} (h-in);
    \draw[wire] (f-out) to node[midway, left, xshift=-1mm] {$B$} (B);
    \draw[wire] (h-out) to node[midway, right, xshift=1mm] {$D$} (D);
\end{tikzpicture}}}
&
    \vcenter{\hbox{\begin{tikzpicture}
    \coordinate (I) at (-0.5,0);
    \coordinate (A) at (0.5,0);

    \coordinate (mid) at ($ (I)!0.5!(A) $);

    \node[box](rho) at ($ (mid) - (0,\bigvertsep) $) {$\rho_A$};
    \AddPorts{rho}{basic/l1/, basic/l2/};

    \coordinate (A') at ($ (rho.south) - (0,\bigvertsep) $);

    \draw[wire] (I) to node[midway, left, xshift=-1mm] {$I$} (l1-in);
    \draw[wire] (A) to node[midway, right, xshift=1mm] {$A$} (l2-in);
    \draw[wire] (rho.south) to node[midway, right, xshift=1mm] {$A$} (A');
\end{tikzpicture}}}
\end{array}
\]

The compatibility between serial and parallel compositions is governed by the \emph{interchange law}:
\begin{equation}
\label{eq:interchange}
\begin{array}{c}
    k: D \to E \\
    (g \circ f) \otimes (k \circ h) = (g \otimes k) \circ (f \otimes h),
\end{array}
\end{equation}
which is expressed in the following string diagrams:
\[
\begin{array}{c@{\quad\quad=\quad\quad}c}
    \vcenter{\hbox{\begin{tikzpicture}
    \coordinate (A) at (-0.5,0);
    \node[box](f) at ($ (A) - (0,\bigvertsep) $) {$f$};
    \AddPorts{f}{basic/f1/};
    \node[box](g) at ($ (f.south) - (0,\vertsep) $) {$g$};
    \AddPorts{g}{basic/g1/};
    \coordinate (C') at ($ (g.south) - (0,\bigvertsep) $);

    \coordinate (C) at (0.5,0);
    \node[box](h) at ($ (C) - (0,\bigvertsep) $) {$h$};
    \AddPorts{h}{basic/h1/};
    \node[box](k) at ($ (h.south) - (0,\vertsep) $) {$k$};
    \AddPorts{k}{basic/k1/};
    \coordinate (E) at ($ (k.south) - (0,\bigvertsep) $);

    \draw[wire] (A) to node[midway, left, xshift=-1mm] {$A$} (f1-in);
    \draw[wire] (f1-out) to (g1-in);
    \draw[wire] (g1-out) to node[midway, left, xshift=-1mm] {$C$} (C');

    \draw[wire] (C) to node[midway, right, xshift=1mm] {$C$} (h1-in);
    \draw[wire] (h1-out) to (k1-in);
    \draw[wire] (k1-out) to node[midway, right, xshift=1mm] {$E$} (E);
\end{tikzpicture}}}
&
    \vcenter{\hbox{\begin{tikzpicture}
    \coordinate (A) at (-0.5,0);
    \coordinate (C) at (0.5,0);

    \node[box](fh) at (0,-\bigvertsep) {$f \otimes h$};
    \AddPorts{fh}{basic/fh1/,basic/fh2/};

    \node[box](gk) at ($ (fh.south) - (0,\vertsep) $) {$g \otimes k$};
    \AddPorts{gk}{basic/gk1/,basic/gk2/};

    \coordinate (lower) at ($ (gk.south) - (0,\bigvertsep) $);
    \coordinate (C') at (A|-lower);
    \coordinate (E) at (C|-lower);

    \draw[wire] (A) to node[midway, left, xshift=-1mm]{$A$} (fh1-in);
    \draw[wire] (C) to node[midway, right, xshift=1mm]{$C$} (fh2-in);
    \draw[wire] (fh1-out) to (gk1-in);
    \draw[wire] (fh2-out) to (gk2-in);
    \draw[wire] (gk1-out) to node[midway, left, xshift=-1mm]{$C$} (C');
    \draw[wire] (gk2-out) to node[midway, right, xshift=1mm]{$E$} (E);
\end{tikzpicture}}}
\end{array}
\]

\subsection{State Transition Systems}
\label{subsec:prelims:state_transitions}

To bridge the gap from abstract structure to autonomous reasoning, we ground our categorical model in a planning state transition model \cite{automated_planning_2004}.
A state transition system is defined as a tuple $\Sigma = \langle \States, \Actions, \gamma \rangle$, where:
\begin{itemize}
    \item $\States$ is the set of all possible system states;
    \item $\Actions$ is the set of all possible actions that transition one state into another, and;
    \item $\gamma$ is a transition function $\gamma: \States \times \Actions \to \States$ that maps a state and action to another state.
\end{itemize}
A planning problem is expressed as $\mathcal{P} = (\Sigma, \state_\init, \States_\goal)$, and a plan $\pi = \langle \action_0, \action_1, \action_2, \dots \rangle,\; \action_i \in \Actions$ is a sequence of actions.
$\pi$ is a solution to $\mathcal{P}$ if it transitions the initial state $\state_\init$ into a valid goal state $\state_\goal \in \States_\goal$, or $\state_\goal \subseteq \gamma(\state_\init, \pi)$.

\section{Approach}
\label{sec:approach}

We propose a compositional framework that treats robotic system integration as a single, continuous synthesis problem.
By modeling the robot as an abstract circuit within a strict symmetric monoidal category, we can reason over heterogeneous mechanical, electrical, software, and behavioral subsystems in a unified manner.
This section details the mapping of resources to categorical objects, the use of functors to project the system into specific lifecycle views, the translation of this structure into an SMT-based optimization problem, and the family of lifelong queries that the resulting formulation supports.

\subsection{The Robot-Circuit Model}
We define our domain $\cat$ as the free strict symmetric monoidal category generated by a set of atomic resource types $\Resources$.
This allows us to represent the robot's state as a string diagram capturing the consumption and production of resources through state transitions.

\subsubsection{Composite Resource Strings as Objects}
Following the definitions in \Cref{sec:preliminaries}, objects $\obj$ represent the states of our system.
Specifically, a state is the tensor product of atomic resources $\state = \resource_1 \otimes \resource_2 \otimes \dots \otimes \resource_n$ where $\resource_i \in \Resources$.
While a key contribution of this work is the uniform treatment of all system components as undifferentiated categorical structures, representing every atomic resource individually would render diagrams unwieldy.
Therefore, for visual and conceptual clarity in this paper, we group resources into four broad partial states:
\begin{itemize}
    \item \emph{Electromechanical:}
    \begin{equation}
        \vspace{-1ex}
        \begin{array}{c@{\;\;\raisebox{0.2em}{$=$}\;\;}c@{\;\;\raisebox{0.2em}{$\otimes$}\;\;}c@{\;\;\raisebox{0.2em}{$\otimes$}\;\;}c@{\;\;\raisebox{0.2em}{$\otimes$}\;\;}c}
    \shortstack[c]{%
      \raisebox{-.15cm}{\includesvg[height=0.5cm]{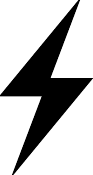}}\\[-0.2em]
      \ssw{EM}
    }
&
    \shortstack[c]{%
      \raisebox{-.15cm}{\includesvg[height=0.5cm]{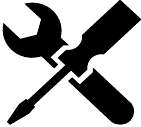}}\\[-0.2em]
      \ssw{Hardware}
    }
&
    \shortstack[c]{%
      \raisebox{-.15cm}{\includesvg[height=0.5cm]{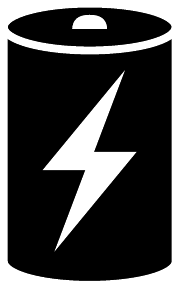}}\\[-0.2em]
      \ssw{Power}
    }
&
    \shortstack[c]{%
      \raisebox{-.15cm}{\includesvg[height=0.35cm]{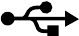}}\\[-0.1em]
      \ssw{Ports}
    }
&
    \dots
\end{array}

    \end{equation}
    \item \emph{Computational:}
    \begin{equation}
        \vspace{-1ex}
        \begin{array}{c@{\;\;\raisebox{0.2em}{$=$}\;\;}c@{\;\;\raisebox{0.2em}{$\otimes$}\;\;}c@{\;\;\raisebox{0.2em}{$\otimes$}\;\;}c@{\;\;\raisebox{0.2em}{$\otimes$}\;\;}c@{\;\;\raisebox{0.2em}{$\otimes$}\;\;}c}
    \shortstack[c]{%
      \raisebox{-.15cm}{\includesvg[height=0.5cm]{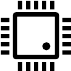}}\\[-0.2em]
      \ssw{Compute}
    }
&
    \shortstack[c]{%
      \raisebox{-.15cm}{\includesvg[height=0.5cm]{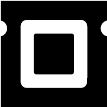}}\\[-0.2em]
      \ssw{CPU}
    }
&
    \shortstack[c]{%
      \raisebox{-.15cm}{\includesvg[height=0.35cm]{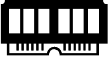}}\\[-0.1em]
      \ssw{GPU}
    }
&
    \shortstack[c]{%
      \raisebox{-.15cm}{\includesvg[height=0.5cm]{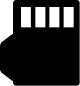}}\\[-0.2em]
      \ssw{Disk}
    }
&
    \shortstack[c]{%
      \raisebox{-.15cm}{\includesvg[height=0.35cm]{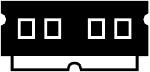}}\\[-0.1em]
      \ssw{RAM}
    }
&
    \dots
\end{array}

    \end{equation}
    \item \emph{Informational:}
    \begin{equation}
        \vspace{-1ex}
        \begin{array}{c@{\;\;\raisebox{0.2em}{$=$}\;\;}c@{\;\;\raisebox{0.2em}{$\otimes$}\;\;}c@{\;\;\raisebox{0.2em}{$\otimes$}\;\;}c@{\;\;\raisebox{0.2em}{$\otimes$}\;\;}c}
    \shortstack[c]{%
      \raisebox{-.15cm}{\includesvg[height=0.45cm]{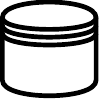}}\\[-0.2em]
      \ssw{Info}
    }
&
    \shortstack[c]{%
      \raisebox{-.15cm}{\includesvg[height=0.5cm]{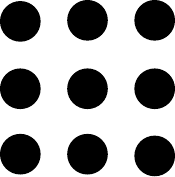}}\\[-0.18em]
      \ssw{ROS}
    }
&
    \shortstack[c]{%
      \raisebox{-.15cm}{\includesvg[height=0.5cm]{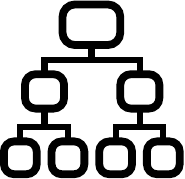}}\\[-0.2em]
      \ssw{BT.CPP}
    }
&
    \shortstack[c]{%
      \raisebox{-.15cm}{\includesvg[height=0.5cm]{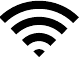}}\\[-0.2em]
      \ssw{Wi-Fi}
    }
&
    \dots
\end{array}

    \end{equation}
    \item \emph{External:}
    \begin{equation}
        \vspace{-1ex}
        \begin{array}{c@{\;\;=\;\;}c@{\;\;\otimes\;\;}c@{\;\;\otimes\;\;}c@{\;\;\otimes\;\;}c}
    \shortstack[c]{%
      \raisebox{-.15cm}{\includesvg[height=0.65cm]{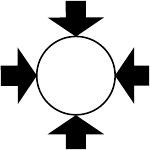}}\\[-0.2em]
      \ssw{External}
    }
&
    \shortstack[c]{%
      \raisebox{-.15cm}{\includesvg[height=0.65cm]{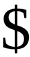}}\\[-0.4em]
      \ssw{Cost}
    }
&
    \shortstack[c]{%
      \raisebox{-.15cm}{\includesvg[height=0.5cm]{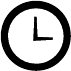}}\\[-0.2em]
      \ssw{Time}
    }
&
    \shortstack[c]{%
      \raisebox{-.15cm}{\includesvg[height=0.5cm]{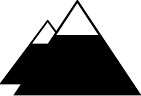}}\\[-0.2em]
      \ssw{Environment}
    }
&
    \dots
\end{array}

    \end{equation}
\end{itemize}
The complete system state $\state$ is then expressed as the tensor product of these partial states:
\begin{equation}
    \begin{array}{c@{\;\;\raisebox{0.2em}{$=$}\;\;}c@{\;\;\raisebox{0.2em}{$\otimes$}\;\;}c@{\;\;\raisebox{0.2em}{$\otimes$}\;\;}c@{\;\;\raisebox{0.2em}{$\otimes$}\;\;}c}
    \raisebox{0.2em}{$\state$}
&
    \shortstack[c]{%
      \raisebox{-.15cm}{\includesvg[height=0.5cm]{figures/images/ct/power.svg}}\\[-0.2em]
      \ssw{EM}
    }
&
    \shortstack[c]{%
      \raisebox{-.15cm}{\includesvg[height=0.5cm]{figures/images/ct/cpu.svg}}\\[-0.2em]
      \ssw{Compute}
    }
&
    \shortstack[c]{%
      \raisebox{-.15cm}{\includesvg[height=0.5cm]{figures/images/ct/datastore.svg}}\\[-0.2em]
      \ssw{Info}
    }
&
    \shortstack[c]{%
      \raisebox{-.15cm}{\includesvg[height=0.65cm]{figures/images/ct/external.svg}}\\[-0.2em]
      \ssw{External}
    }
\end{array}
.
\end{equation}
Notably, because $\cat$ is symmetric, the ordering of these atomic resources within $\state$ does not affect the representation.

\subsubsection{State Transitions as Morphisms}
Morphisms $\morph$ $\action: \state_\mathrm{in} \to \state_\mathrm{out}$ constitute the action space $\Actions$ within our state transition system.
This generalizes the many types of system mutations that could be expected in a reconfigurable robot, including changes to physical hardware, running software, and plan execution, into a unified definition.
\begin{equation}
\label{eq:example_morphism}
\begin{tikzpicture}[
        baseline=(current bounding box.center),
    ]
    \pgfdeclarelayer{background}
    \pgfsetlayers{background,main}
    
    \coordinate (root) at (0,0);

    \node[right=0cm of root](em) {
        \shortstack[c]{%
            \raisebox{-.15cm}{\includesvg[height=0.5cm]{figures/images/ct/power.svg}}\\[-0.2em]
        }%
    };
    \node[right=0cm of em](comp) {
        \shortstack[c]{%
            \raisebox{-.15cm}{\includesvg[height=0.5cm]{figures/images/ct/cpu.svg}}\\[-0.2em]
        }%
    };
    \node[right=0cm of comp](info) {
        \shortstack[c]{%
            \raisebox{-.15cm}{\includesvg[height=0.5cm]{figures/images/ct/datastore.svg}}\\[-0.2em]
        }%
    };
    \node[right=0cm of info](ext) {
        \shortstack[c]{%
            \raisebox{-.15cm}{\includesvg[height=0.5cm]{figures/images/ct/external.svg}}\\[-0.2em]
        }%
    };

    \coordinate (mid) at ($ (em.south)!0.5!(ext.south) $);
    \coordinate (lvl1) at ($ (mid) - (0,\vertsep)$);

    \node[box] (driver) at (mid |- lvl1) {%
        \shortstack[c]{%
            \scriptsize\sw{CameraDriver}%
        }%
    };
    \AddPorts{driver}{basic/em/, basic/comp/, basic/info/, basic/ext/}

    \coordinate (lvl2) at ($ (driver.south) - (0,\vertsep)$);
    \node[anchor=north](em') at (em |- lvl2) {
        \raisebox{-.15cm}{\includesvg[height=0.5cm]{figures/images/ct/power.svg}}
    };
    \node[anchor=north](comp') at (comp |- lvl2) {
        \raisebox{-.15cm}{\includesvg[height=0.5cm]{figures/images/ct/cpu.svg}}
    };
    \node[anchor=north](info') at (info |- lvl2) {
        \raisebox{-.15cm}{\includesvg[height=0.5cm]{figures/images/ct/datastore.svg}}
    };
    \node[anchor=north](ext') at (ext |- lvl2) {
        \raisebox{-.15cm}{\includesvg[height=0.5cm]{figures/images/ct/external.svg}}
    };

    \begin{pgfonlayer}{background}
        \draw[wire] (em) to (em-in);
        \draw[wire] (em-out) to (em');
        \draw[wire] (comp) to (comp-in);
        \draw[wire] (comp-out) to (comp');
        \draw[wire] (info) to (info-in);
        \draw[wire] (info-out) to (info');
        \draw[wire] (ext) to (ext-in);
        \draw[wire] (ext-out) to (ext');
    \end{pgfonlayer}

    \coordinate(blowout) at ($(driver.east) + (1.0,0.0)$);
    \node[font=\scriptsize, align=left, fill=gray!10, anchor=west, outer sep=0pt] (note) at (blowout) {
        \textbf{Preconditions}\\
        \textbullet\ $\sw{Hardware} \ni \sw{Camera}$ \\
        \textbullet\ $\sw{RAM} \geq \sw{RAM}_\text{req}$\\
        \textbullet\ \dots\\
        \textbf{Postconditions}\\
        \textbullet\ $\sw{RAM} := \sw{RAM} - \sw{RAM}_\text{req}$\\
        \textbullet\ $\sw{ROS} := \sw{ROS} \cup \sw{\{ImageTopic\}}$\\
        \textbullet\ \dots
    };
    \fill[gray!10]
      (driver.east |- driver.north)
      -- (note.west |- note.north)
      -- (note.west |- note.south)
      -- (driver.east |- driver.south)
      -- cycle;
    
\end{tikzpicture}
\end{equation}
As expressed in \Cref{eq:example_morphism}, each morphism may contain a set of preconditions that query the input resources and must hold for the morphism to be applicable at that state, and a set of postconditions that produce updated system resources if that morphism is applied.
Composite systems are formed by applying serial ($\action_2 \circ \action_1$) and parallel ($\action_1 \otimes \action_2$) compositions of atomic morphisms indexed by discrete steps $t_1, \dots, t_n$:
\begin{equation}
\label{eq:example_composition}
\begin{array}{c@{\qquad}c}
    \begin{tikzpicture}[
        baseline=(current bounding box.center),
    ]
    \pgfdeclarelayer{background}
    \pgfsetlayers{background,main}
    
    \coordinate (root) at (0,0);

    \node[right=0cm of root](em) {
        \shortstack[c]{%
            \raisebox{-.15cm}{\includesvg[height=0.5cm]{figures/images/ct/power.svg}}\\[-0.2em]
        }%
    };
    \node[right=0cm of em](comp) {
        \shortstack[c]{%
            \raisebox{-.15cm}{\includesvg[height=0.5cm]{figures/images/ct/cpu.svg}}\\[-0.2em]
        }%
    };
    \node[right=0cm of comp](info) {
        \shortstack[c]{%
            \raisebox{-.15cm}{\includesvg[height=0.5cm]{figures/images/ct/datastore.svg}}\\[-0.2em]
        }%
    };
    \node[right=0cm of info](ext) {
        \shortstack[c]{%
            \raisebox{-.15cm}{\includesvg[height=0.5cm]{figures/images/ct/external.svg}}\\[-0.2em]
        }%
    };

    \coordinate (mid) at ($ (em.south)!0.5!(ext.south) $);
    \coordinate (lvl1) at ($ (mid) - (0,\vertsep)$);

    \node[box] (camera) at (mid |- lvl1) {%
        \shortstack[c]{%
            \ssw{Camera}%
        }%
    };
    \AddPorts{camera}{readfwrite/camera-em/, readwrite/camera-ext/}

    \coordinate (lvl2) at ($ (camera.south) - (0,\vertsep)$);
    \node[box] (camera_driver) at (mid |- lvl2) {%
        \shortstack[c]{%
            \ssw{CameraDriver}%
        }%
    };
    \AddPorts{camera_driver}{read/camera_driver-em/, readwrite/camera_driver-comp/, fwrite/camera_driver-info/, readwrite/camera_driver-ext/}

    \coordinate (lvl3) at ($ (camera_driver.south) - (0,\vertsep)$);
    \coordinate (mid) at ($ (info.south)!0.5!(ext.south) $);
    \node[box] (take_photo) at (mid |- lvl3) {%
        \shortstack[c]{%
            \ssw{TakePhoto}%
        }%
    };
    \AddPorts{take_photo}{freadfwrite/take_photo-info/, readwrite/take_photo-ext/}

    \coordinate (lvl4) at ($ (take_photo.south) - (0,\vertsep)$);
    \node[anchor=north](em') at (em |- lvl4) {
        \raisebox{-.15cm}{\includesvg[height=0.5cm]{figures/images/ct/power.svg}}
    };
    \node[anchor=north](comp') at (comp |- lvl4) {
        \raisebox{-.15cm}{\includesvg[height=0.5cm]{figures/images/ct/cpu.svg}}
    };
    \node[anchor=north](info') at (info |- lvl4) {
        \raisebox{-.15cm}{\includesvg[height=0.5cm]{figures/images/ct/datastore.svg}}
    };
    \node[anchor=north](ext') at (ext |- lvl4) {
        \raisebox{-.15cm}{\includesvg[height=0.5cm]{figures/images/ct/external.svg}}
    };

    \begin{pgfonlayer}{background}
        \draw[wire] (em) to (camera-em-in);
        \draw[wire] (ext) to (camera-ext-in);
        \draw[wire] (camera-em-out) to (camera_driver-em-in);
        \draw[wire] (camera-ext-out) to (camera_driver-ext-in);
        \draw[wire] (comp) to (camera_driver-comp-in);
        \draw[wire] (info) to (camera_driver-info-in);
        \draw[wire] (camera_driver-em-out) to (em');
        \draw[wire] (camera_driver-comp-out) to (comp');
        \draw[wire] (camera_driver-info-out) to (take_photo-info-in);
        \draw[wire] (camera_driver-ext-out) to (take_photo-ext-in);
        \draw[wire] (take_photo-info-out) to (info');
        \draw[wire] (take_photo-ext-out) to (ext');
    \end{pgfonlayer}

    \coordinate(t) at (-0.5, 0);
    \node(t1) at (t |- camera) {$t_1$};
    \node(t2) at (t |- camera_driver) {$t_2$};
    \node(t3) at (t |- take_photo) {$t_3$};
    \draw[dotted] (t1) -- (camera.west);
    \draw[dotted] (t2) -- (camera_driver.west);
    \draw[dotted] (t3) -- (take_photo.west);

\end{tikzpicture}
&
    \begin{tikzpicture}[
    baseline=(current bounding box.center),
    every node/.style={font=\scriptsize}
]

\pgfdeclarelayer{background}
\pgfdeclarelayer{last}
\pgfsetlayers{last,background,main}


\coordinate (lvl) at ($ (0,0) - (0,1.75*\vertsep)$);
\node[] (p1) at (lvl) {};
\AddPorts{p1}{read/1/}
\node[expl, right=0.3cm of p1.east, yshift=5pt](expl1) {Precondition};

\coordinate (lvl) at ($ (p1) - (0,\vertsep)$);
\node[] (p2) at (lvl) {};
\AddPorts{p2}{write/1/}
\node[expl, right=0.3cm of p2.east, yshift=2pt](expl2) {Postcondition};

\coordinate (lvl) at ($ (p2) - (0,1.75*\vertsep)$);
\node[](p3_1) at ($ (lvl) + (-0.2,0) $) {};
\node[](p3_2) at ($ (lvl) + (0.2,0) $) {};
\AddPorts{p3_1}{fread/1/}
\AddPorts{p3_2}{fwrite/p2/}

\node[expl, yshift=3.5pt](expl3) at (expl2.west |- p3_2) {Free Parameters\\within Conditions};


\begin{pgfonlayer}{last}
    \node[
      fit=(p1) (expl1) (p2) (expl2) (p3_1) (p3_2) (expl3),
      draw,
      fill=gray!10,
      rounded corners,
      label=north:{\textbf{Port Legend}},
      inner sep=2pt
    ] {};
\end{pgfonlayer}

\end{tikzpicture}
\end{array}
\end{equation}
We use port symbols shown in \Cref{eq:example_composition} to indicate whether a morphism imposes preconditions and/or postconditions on an input or output resource in the remaining string diagrams in this paper.
The presence of free parameters within preconditions and postconditions is further discussed in \Cref{subsec:approach:satisfiability}.

\subsubsection{Planning Objective}
With the states and actions in our state transition system mapped to categorical representations, the objective becomes to find a valid composition of morphisms from $\Actions$ using $\circ$ and $\otimes$ that transforms an initial state $\state_\init$ into a goal state $\state_\goal \in \States_\goal$.
Formally stated, we solve if
\begin{equation}
    \exists \pi \in \langle \Actions \rangle_{\circ, \otimes} \quad \text{s.t.}\;\; \pi: \state_\init \to \state_\goal \in \States_\goal.
\end{equation}

\subsection{Contextual Reasoning via Projection Functors}
While it is powerful to capture the robot's entire complex state in a unified representation, it is often practical to consider only some reduced version of that state that removes elements unnecessary for the current analysis.
For example, considering the available hardware interfaces on a robot is critical for physical reconfiguration, but is unnecessary when solving for a new software state over a fixed physical morphology.
We formalize this ``view-switching'' using \emph{monoidal functors} $F_\view: \cat \to \cat_\view$ that project the overall system category $\cat$ into view-specific categories $\cat_\view$.

Let $\view$ be a specific lifecycle stage (e.g. software installation).
If a resource $\resource$ (like available hardware interfaces) is irrelevant to that stage, the functor $F_\view$ maps it to the monoidal unit $I$: $F_\view(\resource) = I$.
Because $\state = I \otimes \resource_2 = \resource_2$ from \Cref{eq:ct:strict_properties}, this effectively removes the resource $\resource_1$ and its associated constraints from $\state$ without altering the underlying topological structure of the composition.

\begin{equation}
\begin{array}{c@{\quad}c@{\quad}c}
    \begin{tikzpicture}[
        baseline=(current bounding box.center),
    ]
    \pgfmathsetmacro{\opacity}{0.07}

    \pgfdeclarelayer{background}
    \pgfsetlayers{background,main}
    
    \coordinate (root) at (0,0);
    \node[right=0cm of root](power) {
        \shortstack[c]{%
            \raisebox{-.15cm}{\includesvg[height=0.5cm]{figures/images/ct/power.svg}}\\[-0.2em]
        }%
    };
    \node[right=0cm of power](compute) {
        \shortstack[c]{%
            \raisebox{-.15cm}{\includesvg[height=0.5cm]{figures/images/ct/cpu.svg}}\\[-0.2em]
        }%
    };
    \begin{scope}[opacity=\opacity]
    \node[right=0cm of compute](data) {
        \shortstack[c]{%
            \raisebox{-.15cm}{\includesvg[height=0.5cm]{figures/images/ct/datastore.svg}}\\[-0.2em]
        }%
    };
    \end{scope}
    \node[right=0cm of data](external) {
        \shortstack[c]{%
            \raisebox{-.15cm}{\includesvg[height=0.5cm]{figures/images/ct/external.svg}}\\[-0.2em]
        }%
    };

    \coordinate(mid) at ($ (power.south)!0.475!(external.south) $);
    \node[box] (robot) at ($(mid) - (0,\vertsep)$) {%
        \shortstack[c]{%
            \ssw{Robot}%
        }%
    };
    \AddPorts{robot}{freadfwrite/q-power/,freadfwrite/q-compute/,basic/q-data/,freadfwrite/q-external/}

    \coordinate (lvl) at ($ (robot.south) - (0,\vertsep)$);
    \node[anchor=north](power') at (power |- lvl) {
        \raisebox{-.15cm}{\includesvg[height=0.5cm]{figures/images/ct/power.svg}}
    };
    \node[anchor=north](compute') at (compute |- lvl) {
        \raisebox{-.15cm}{\includesvg[height=0.5cm]{figures/images/ct/cpu.svg}}
    };
    \begin{scope}[opacity=\opacity]
    \node[anchor=north](data') at (data |- lvl) {
        \raisebox{-.15cm}{\includesvg[height=0.5cm]{figures/images/ct/datastore.svg}}
    };
    \end{scope}
    \node[anchor=north](external') at (external |- lvl) {
        \raisebox{-.15cm}{\includesvg[height=0.5cm]{figures/images/ct/external.svg}}
    };

    \begin{pgfonlayer}{background}
        \draw[wire] (power) to (q-power-in);
        \draw[wire] (q-power-out) to (power');
        \draw[wire] (compute) to (q-compute-in);
        \draw[wire] (q-compute-out) to (compute');
        \begin{scope}[opacity=\opacity]
        \draw[wire] (data) to (q-data-in);
        \draw[wire] (q-data-out) to (data');
        \end{scope}
        \draw[wire] (external) to (q-external-in);
        \draw[wire] (q-external-out) to (external');
    \end{pgfonlayer}

\end{tikzpicture}
&
    \begin{tikzpicture}[
        baseline=(current bounding box.center),
    ]
    \pgfmathsetmacro{\opacity}{0.07}

    \pgfdeclarelayer{background}
    \pgfsetlayers{background,main}
    
    \coordinate (root) at (0,0);
    \begin{scope}[opacity=\opacity]
    \node[right=0cm of root](power) {
        \shortstack[c]{%
            \raisebox{-.15cm}{\includesvg[height=0.5cm]{figures/images/ct/power.svg}}\\[-0.2em]
        }%
    };
    \end{scope}
    \node[right=0cm of power](compute) {
        \shortstack[c]{%
            \raisebox{-.15cm}{\includesvg[height=0.5cm]{figures/images/ct/cpu.svg}}\\[-0.2em]
        }%
    };
    \node[right=0cm of compute](data) {
        \shortstack[c]{%
            \raisebox{-.15cm}{\includesvg[height=0.5cm]{figures/images/ct/datastore.svg}}\\[-0.2em]
        }%
    };
    \begin{scope}[opacity=\opacity]
    \node[right=0cm of data](external) {
        \shortstack[c]{%
            \raisebox{-.15cm}{\includesvg[height=0.5cm]{figures/images/ct/external.svg}}\\[-0.2em]
        }%
    };
    \end{scope}

    \coordinate(mid) at ($ (power.south)!0.5!(external.south) $);
    \node[box] (robot) at ($(mid) - (0,\vertsep)$) {%
        \shortstack[c]{%
            \ssw{Robot}%
        }%
    };
    \AddPorts{robot}{basic/q-power/,freadfwrite/q-compute/,freadfwrite/q-data/,basic/q-external/}

    \coordinate (lvl) at ($ (robot.south) - (0,\vertsep)$);
    \begin{scope}[opacity=\opacity]
    \node[anchor=north](power') at (power |- lvl) {
        \raisebox{-.15cm}{\includesvg[height=0.5cm]{figures/images/ct/power.svg}}
    };
    \end{scope}
    \node[anchor=north](compute') at (compute |- lvl) {
        \raisebox{-.15cm}{\includesvg[height=0.5cm]{figures/images/ct/cpu.svg}}
    };
    \node[anchor=north](data') at (data |- lvl) {
        \raisebox{-.15cm}{\includesvg[height=0.5cm]{figures/images/ct/datastore.svg}}
    };
    \begin{scope}[opacity=\opacity]
    \node[anchor=north](external') at (external |- lvl) {
        \raisebox{-.15cm}{\includesvg[height=0.5cm]{figures/images/ct/external.svg}}
    };
    \end{scope}

    \begin{pgfonlayer}{background}
        \begin{scope}[opacity=\opacity]
        \draw[wire] (power) to (q-power-in);
        \draw[wire] (q-power-out) to (power');
        \draw[wire] (external) to (q-external-in);
        \draw[wire] (q-external-out) to (external');
        \end{scope}
        \draw[wire] (compute) to (q-compute-in);
        \draw[wire] (q-compute-out) to (compute');
        \draw[wire] (data) to (q-data-in);
        \draw[wire] (q-data-out) to (data');
    \end{pgfonlayer}

\end{tikzpicture}
\end{array}
\end{equation}

This allows the same high-level system model to be solved, verified, or optimized against multiple distinct views by simply swapping the projection functor.
While our approach is capable of solving for the entire end-to-end system composition, these functors allow the state to be collapsed into any subset of desired resources.
This effectively allows the framework to emulate traditionally decoupled stages, such as hardware design, software integration, or behavior planning, by narrowing the focus of the solver to only the resources relevant during each stage.

\subsection{Mapping to Satisfiability with Free Variable Synthesis}
\label{subsec:approach:satisfiability}

To generate a valid planning solution $\pi$, we map the categorical model to an SMT problem.

\subsubsection{The Logic Formula}
An initial state $\state_\init$ is parameterized by a set of literals $l \in \state_\init$.
Likewise, the space of possible goal states is parameterized by $l \in \States_\goal$.
For each morphism $\action \in \Actions$, we capture the atomic literals within its preconditions $l \in \pre_\action$ and postconditions $l \in \post_\action$.
Let the Boolean domain be $\mathbb{B} := \{0,1\}$.
Over a planning horizon $k$, we search for an assignment to Boolean \emph{activation variables} $\activation \in \Activations$ where $\Activations \in \mathbb{B}^{|\Actions| \times k}$ within the formula
\begin{equation}
\label{eq:logic_formula}
    \begin{array}{c}
         T(t) = \displaystyle\bigwedge_{\action \in \Actions} \left( \activation_\action(t) \Rightarrow \left( \displaystyle\bigwedge_{\literal \in \pre_\action} \literal(t) \land \displaystyle\bigwedge_{\literal \in \post_\action} \literal(t+1) \right) \right) \\
         \formula = \displaystyle\bigwedge_{\literal \in \state_\init} \literal(0) \land \displaystyle\bigwedge_{0 \leq t < k} T(t) \land \displaystyle\bigwedge_{\literal \in \States_\goal} l(k)
    \end{array}
\end{equation}
such that the resulting formula is satisfied ($\formula \in \sat$).
$T(t)$ encodes the transition logic, and an activation $\activation_\action(t)$ is only valid if all literals $l \in \pre_\action$ are satisfied at the step $t$ and subsequently applies all literals $l \in \post_\action$ at $t+1$.
Frame axioms additionally enforce that any literal not appearing as a postcondition of an activated action is retained across the transition.
Formally, for the set of all literals $\Literals$ and time step $t$:
\begin{equation}
    F(t) = \displaystyle\bigwedge_{\literal \in \Literals} \left( \left( \literal(t) \neq \literal(t+1) \right) \implies \displaystyle\bigvee_{\substack{\action \in \Actions \\ \literal \in \post_\action}} \activation_\action(t) \right).
\end{equation}

\subsubsection{Free Variable Synthesis}
We leverage the ability of SMT to reason over more complex datatypes beyond Booleans to synthesize missing information influencing the composition.
Thus, in addition to the Boolean activation variables $\activation \in \Activations$, we simultaneously reason over two types of free variables:
\begin{itemize}
    \item \emph{Parametric free variables} ($\param \in \Params$):
    These represent free variables within the preconditions and/or postconditions of morphisms.
    For example, rather than defining a strict amount of \sw{RAM} consumption for a specific piece of software in the precondition literal $\sw{RAM} \geq \sw{RAM}_\req$, $\sw{RAM}_\req$ can be left as a free variable $\param_\req$ that is simultaneously solved within the composition. Because these parameters vary in type based on the underlying resource, we define the heterogeneous tuple $\Params = (\param_1, \dots, \param_n)$ where each $\param_i$ is drawn from a specific SMT sort $\mathcal{D}_i \in \{\mathbb{R}, \mathbb{Z}, \mathbb{B}, \mathrm{String}, \mathrm{Enum}, \mathrm{Set}, \dots \}$.
    In this way, parametric free variables can be used to solve for design-time targets within morphisms or can be left free into deployment to solve for composition-dependent instantiations including physical cable connections, software parameters, or the names of instances of data in memory, for example.
    \item \emph{Structural free variables} ($\struct \in \Structs$):
    These represent entire morphisms within a composition that are left free.
    If a composition is invalid because of unsatisfiable preconditions, a structural free variable $\struct$ can be introduced to represent a hypothetical component and its preconditions and postconditions solved such that the overall composition is satisfied.
    In this way, structural free variables can be used to create design targets for morphisms that, if they existed, would render a composition valid.
    The space of structural free variables $\Structs$ is constrained by the current projection $F_\view$.
\end{itemize}

\subsubsection{Finite Grounding}
\label{subsubsec:approach:grounding}

The parametric free variables $\param \in \Params$ can range over symbolic sorts whose vocabularies are \emph{open}, including \sw{String} and \sw{Set} types.
For example, the name of a ROS topic, the identifier of an instance of data in memory, or the class of a detectable object using an open-vocabulary model \cite{yoloe_2025} should not be restricted to a finite set of values specified at design time.
Pre-enumeration of all possible variable assignments within the domain is an expensive and brittle process in the best case, and in the context of a recompositional robot whose structure is frequently changing, it is often impossible to know the full set of possible values at design time.
In practice, our solver also supports shell execution within condition evaluation, and arbitrary code execution means that the set of possible values for a variable may be infinite.
Reasoning over an infinite domain, however, is generally undecidable.
Our approach therefore grounds each symbolic sort over a finite universe $H$ built from the symbolic constants occurring across the design-time domain and the generated problem definition along with one fresh reserve element per free variable.
Our encoding only supports equality, disequality, and set membership constraints over these symbolic sorts, which makes the finite grounding sound and improves solver performance by reducing the search space.

\begin{proposition}[Finite Grounding]
\label{prop:grounding}
    Let $\formula$ be a formula of \Cref{eq:logic_formula} that is quantifier-free over its symbolic sorts, whose atoms over those sorts are restricted to equality, disequality, and set membership, and whose set terms are built by only adding or removing single set elements.
    Let $H$ be the set of symbolic constants occurring in $\formula$ extended with one fresh reserve element per free symbolic variable.
    Then, for every universe $U \supseteq H$, $\formula$ is satisfiable over $U$ if and only if it is satisfiable over $H$.
\end{proposition}

\begin{proof}[Proof Sketch]
    The admitted atoms depend on symbolic values only through equality, disequality, and set membership.
    Injectivity carries equality and disequality both ways, and taking images commutes with set membership and with adding or removing single elements.
    Therefore, any injection of the symbolic universe that fixes the constants of $\formula$ preserves and reflects every atom.

    \noindent($\Leftarrow$) The inclusion $H \hookrightarrow U$ is such an injection, so any model over $H$ is also a model over $U$.

    \noindent($\Rightarrow$) In a model over $U$, membership is tested only against the element terms occurring in $\formula$, so we may drop from each set every element unequal to all of them without changing any atom.
    Every symbolic value then is a constant of $\formula$ or the value of a free variable, the latter contributing at most one distinct element each.
    An injection fixing the constants and sending these to distinct reserve elements of $H$ yields a model over $H$. 
\end{proof}

The result of \Cref{prop:grounding} is that a value never seen at modeling time requires no re-enumeration because the grounding is built from the problem and domain at runtime rather than just from the domain at modeling time, meaning admitting an unmodeled value requires only a single problem-file update rather than a model change.
This advantage is discussed in \Cref{subsec:eval:perturbation}.

\subsubsection{Optimization}
Optionally, this can be formulated as an optimization problem, where literals associated with certain resources can be applied as soft constraints to form $\formula_\soft$ and the remaining literals form $\formula_\hard$.
In this case, we solve for activations and free parametric and structural variables that satisfy the hard constraints $\formula_\hard$ while minimizing a cost function.
The cost function is composed of penalties for violating the soft constraints, where each soft constraint $\formula_{\soft_i}$ has an associated violation indicator $z_i$ and weight $w_i$, and an auxiliary objective $c$ that can be declared to maximize a particular functionality or minimize the consumption of a particular resource:
\begin{equation}
\label{eq:optimization}
\begin{aligned}
    x = \arg\min_{\activation, \param, \struct} \quad & c(\activation, \param, \struct) + \sum_{i=1}^m w_i z_i \\
    \text{subject to} \quad
    & \formula_\hard(\activation, \param, \struct) \in \sat \\
    & z_i \iff \neg \formula_{\soft_i}(\activation, \param, \struct), \quad \forall i \in \{1, \dots, m\} \\
    & \activation \in \Activations, \param \in \Params, \struct \in \Structs.
\end{aligned}
\end{equation}

\subsection{Lifelong Queries over the Model}
\label{subsec:approach:queries}

\Cref{eq:logic_formula} captures the composition as a conjunction of labeled constraints, which allows the model to support queries beyond plan existence.
Take $\formula = \bigwedge_{\conjunct \in \Conjuncts} \conjunct$, where $\Conjuncts = \Conjuncts_\init \cup \Conjuncts_{\mathrm{trans}} \cup \Conjuncts_\goal$ is the set of all conjuncts in \Cref{eq:logic_formula} and let a solution be the assignment $x = (\activation, \param, \struct)$ of \Cref{eq:optimization}.
When an update to any component captured in the model replaces $\formula$ with $\formula'$, three questions can be answered, each of which is a minimality query over the same formula:

\subsubsection{Diagnose: why is the composition infeasible?}
When $\formula' \in \unsat$, we explain the failure relative to the composition.
An explanation is an irreducible subset $\mus \subseteq \Conjuncts$ such that
\begin{equation}
\label{eq:diagnose}
    \bigwedge_{\conjunct \in \mus} \conjunct \in \unsat
    \quad \text{and} \quad
    \forall\, \conjunct \in \mus:\;
    \bigwedge_{\psi \in \mus \setminus \{\conjunct\}} \psi \in \sat.
\end{equation}
Given the composition, every constraint in $\mus$ is necessary for the infeasibility, and removing any one restores satisfiability.
$\mus$ is irreducible but may not have minimum cardinality and $\formula'$ may admit several such subsets.

\subsubsection{Relax: what is the minimal restoration?}
Let $\Conjuncts_\req \subseteq \Conjuncts$ be the requirements eligible for weakening.
For each requirement $\rho \in \Conjuncts_\req$ of the form $\rho \equiv (q_\rho \geq \resource_{\req,\rho})$ where $q_\rho$ denotes the current level of a resource governed by $\rho$, introduce a slack variable $\delta_\rho \geq 0$ yielding the weakened literal $\rho(\delta_\rho) \equiv (q_\rho + \delta_\rho \geq \resource_{\req,\rho})$. 
The minimal restoration is then:
\begin{equation}
\label{eq:relax}
\begin{aligned}
    \delta^{*} = \arg\min_{\delta \geq 0} \quad
        & \sum_{\rho \in \Conjuncts_\req} w_\rho\, \delta_\rho \\
    \text{subject to} \quad
        & \bigwedge_{\conjunct \in \Conjuncts \setminus \Conjuncts_\req} \conjunct
          \;\land\;
          \bigwedge_{\rho \in \Conjuncts_\req} \rho(\delta_\rho)
          \;\in\; \sat ,
\end{aligned}
\end{equation}
an instance of \Cref{eq:optimization} in which the weakened requirements form the soft set $\formula_\soft$.
The solver returns both the magnitude $\delta^*$ and, in the discrete case, the minimal set of requirements whose removal restores feasibility.
Importantly, minimality is discovered rather than designated.
\Cref{eq:relax} ranges over all of $\Conjuncts_\req$ simultaneously, whereas the replanning baselines in \Cref{subsec:eval:queries} must re-solve a bisection loop per candidate parameter.

\subsubsection{Recompose: how much must the deployed system change?}
In the recompositional robot problem, the current system state is an incumbent configuration that may be expensive in terms of time, effort, or both, to change.
Thus, when $\formula' \in \unsat$, it is practically useful to incorporate the incumbent state and minimize the cost of required changes when solving for an updated composition. 
Let $\Actions_0 \subseteq \Actions$ be the components of the currently deployed configuration, and let $\Actions_\activation \subseteq \Actions$ denote the configuration induced by a candidate solution.
The recomposition cost of the candidate solution is its weighted churn
\begin{equation}
\label{eq:churn}
    \Delta\!\left(\activation\right)
    = \sum_{\action \,\in\, \Actions_\activation \,\triangle\, \Actions_0}
      w_\action ,
\end{equation}
or the symmetric difference between the candidate and incumbent compositions.
The churn is weighted by each component's replacement cost $w_\action$ to capture the practical costs of reconfiguration, such as mounting a new piece of hardware, connecting cables for power and communication, restarting a piece of software, or modifying a behavior tree.
Minimal-change recomposition then solves the lexicographic program
\begin{equation}
\label{eq:reconfigure}
\begin{aligned}
    x^{*} =
    \arg \operatorname*{lex\,min}_{\activation,\, \param,\, \struct} \quad
    &\bigl( \Delta(\activation),\; c(\activation, \param, \struct) \bigr)
    \\
    \text{subject to} \quad
    &\formula_\hard(\activation, \param, \struct) \in \sat \\
    & \activation \in \Activations, \param \in \Params, \struct \in \Structs ,
\end{aligned}
\end{equation}
which holds the deployment fixed except where an update forces a change, breaking ties by the mission objective $c$ of \Cref{eq:optimization}.
By reversing the priority order to $\operatorname{lex\,min}\bigl(c, \Delta\bigr)$, the same formulation instead recovers the cost-optimal solution regardless of the incumbent configuration.
Agency over the objective prioritization makes the optimality-versus-stability tradeoff an explicit choice that can be determined based on the current context of the system (\Cref{tab:queries}). 

Lifelong recomposition based on these queries takes the form of \Cref{alg:lifelong}, which is used in the demonstration in \Cref{subsec:applications:runtime}.

\begin{algorithm}[t]
\caption{Lifelong recomposition loop (onboard, \Cref{fig:demo})}
\label{alg:lifelong}
\begin{algorithmic}[1]
\Require model, incumbent components $\Actions_0$, projection $F_\view$
\While{mission active}
    \State $\formula' \gets$ encode sensed state, task, catalog \hfill\Cref{eq:logic_formula}
    \If{$\Actions_0$ satisfies $\formula'$} \Comment{composition still valid}
        \State \textbf{continue}
    \ElsIf{$\formula' \in \sat$}
        \State $x^{*} \gets$ minimal-churn re-solve \hfill\Cref{eq:reconfigure}
    \Else
        \State report cause $\mus$ to operator \hfill\Cref{eq:diagnose}
        \State find minimal restoration $\delta^{*}$ \hfill\Cref{eq:relax}
        \State await approval; weaken $\formula' \gets \formula'(\delta^{*})$
        \State $x^{*} \gets$ minimal-churn re-solve \hfill\Cref{eq:reconfigure}
    \EndIf
    \State enact $x^{*}$;\; $\Actions_0 \gets \Actions_\activation$ \Comment{updated composition}
\EndWhile
\end{algorithmic}
\end{algorithm}

\section{Applications}
\label{sec:applications}

We demonstrate the utility of our approach in a multi-stage search-and-rescue application.
We first use our solver as a design-time tool to synthesize specific software and behavioral components and subcompositions, then integrate it into a symbolic robot model running onboard a physical robot that allows it to dynamically recognize and adapt to changing task, environmental, and internal constraints online.

\subsection{Design-Time Component Engineering}
\label{subsec:applications:design}
The search-and-rescue application requires the robot to traverse a facility without a prior map, identify any people within the facility, approach them, and assess their condition before sharing their location and the generated assessment with a remote operator.
Our rudimentary condition assessment approach uses a speech-to-text (STT) model to capture any spoken audio from the person and a vision-language model (VLM) that processes an image captured with the robot's camera and the transcribed text from the STT model to generate a textual description of the person's perceived condition and need for assistance.

Because both the STT model and VLM are expensive to run locally, especially alongside other GPU-native processes, we first use our solver to design a subcomposition of STT model and VLM that can run locally alongside all other software required for the search-and-rescue mission.
For the deployment in \Cref{fig:demo}, we use an NVIDIA Jetson AGX Orin with \SI{64}{\giga\byte} shared RAM/VRAM and a \SI{1}{\tera\byte} disk running on the robot.
To design the software components, we first create the skeleton software components with free parametric variables shown in \Cref{fig:pareto} (top left), for which the composite constraints shown in \Cref{fig:pareto} (top center) are solved.
These constraints provide design parameters for the STT model and VLM that would enable them to run locally on the deployment computer alongside the other required software components.
However, these constraints only inform design targets for software that \emph{can} run on the robot, failing to capture other factors relevant to the deployment including the quality of the generated assessment and time required to generate it.

To address this need and informed by the solved constraints, we consider five variants of the Whisper \cite{whisper_2023} STT model (Tiny, Base, Small, Medium, and Large) as well as two leading open-source VLMs including Gemma4 (E2B, E4B, 12B, 26B, and 31B) \cite{gemma4_2026} and Qwen3.5 (0.8B, 2B, 4B, 9B, 27B, and 35B) \cite{qwen3_2025} and generate morphisms associated with each variant.
As a proxy metric for overall assessment quality, we use a hybrid of score on the MMMU-Pro benchmark \cite{mmmu_pro_2025} for VLM evaluation and word error rate (WER) for STT model evaluation, both of which are metrics provided by the respective model authors.
\begin{equation}
    F_{\mathrm{assessment}} = \mathrm{MMMU} * (1 - \mathrm{WER})
\end{equation}
We also measured the GPU usage and latency of each model processing mission-similar audio, image, and text inputs.

\begin{figure}[t]
\centering
\vspace{1ex}
    \centering
    \begin{tabular}{c@{$\;\to\;\;$}c@{$\;\to\;\;\;$}c}
    \input{diagrams/pareto/stt_vlm}
    &
    \begin{minipage}{0.16\textwidth}
    \input{diagrams/pareto/params}
    \end{minipage}
    &
    \input{diagrams/pareto/whisper_gemma}
    \end{tabular}
    \vspace{1ex}
    
    \includegraphics[width=\columnwidth]{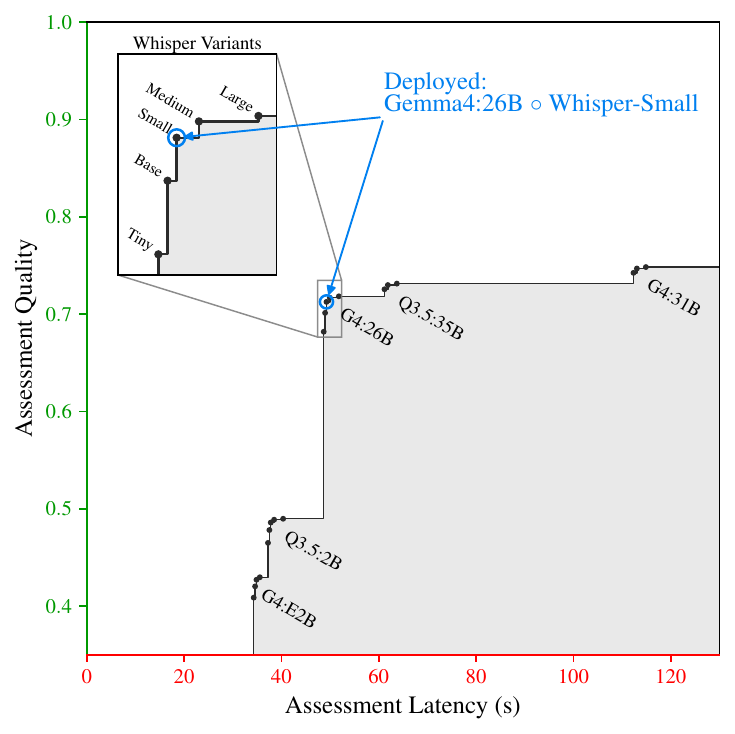}
\vspace{-2ex}
\caption{%
    Development of the combined assessment stack, including hypothetical STT and VLM software components (top left).
    Solving for the free parameters within the subcomposition yields actionable design targets (top center).
    A combination of five STT candidates (Tiny, Base, Small, Medium, and Large variants of Whisper \cite{whisper_2023}) and eleven VLM candidates (E2B, E4B, 12B, 26B, and 31B variants of Gemma4 \cite{gemma4_2026} and 0.8B, 2B, 4B, 9B, 27B, and 35B variants of Qwen3.5 \cite{qwen3_2025}) were evaluated to generate the Pareto front (bottom).
    Of the 55 total compositions, 20 non-dominated compositions form the front.
    Based on these results, the implemented software components use the Small variant of Whisper and the 26B parameter variant of Gemma4, respectively, with both adhering to the standard of \cite{coral_2026} (top right).
}
\label{fig:pareto}
\end{figure}
\begin{figure*}
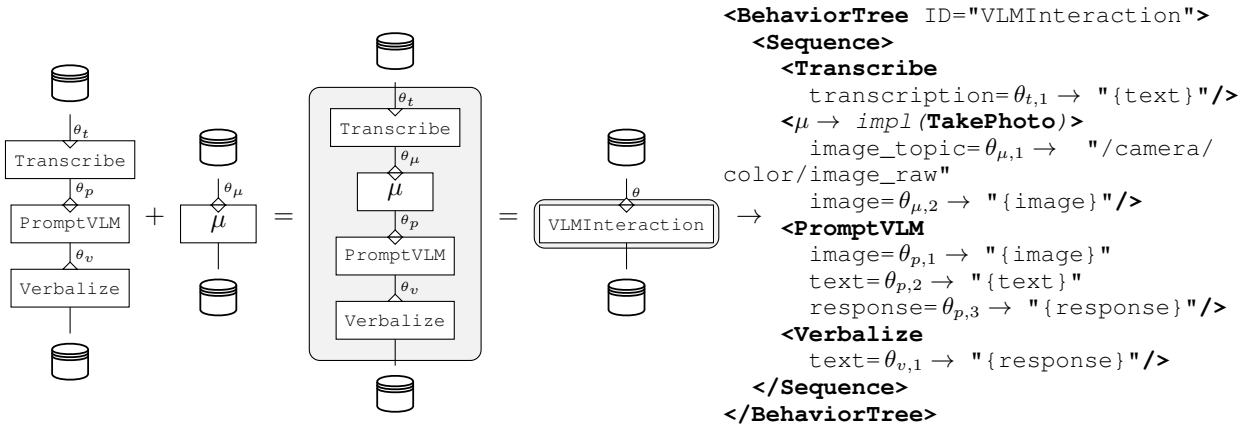

    \centering
    \vspace*{1ex}
    \begin{tabular}{c@{$+\;\;$}c@{$\;=\;\;$}c@{\;$=$\;\;}c@{$\;\to\kern-1.5em$}c}
    \input{diagrams/bt/behavior_tree_free}
    &
    \input{diagrams/bt/take_photo}
    &
    \input{diagrams/bt/behavior_tree_complete}
    &
    \input{diagrams/bt/behavior_tree_composite}
    &
    \begin{minipage}{0.4\textwidth}
        \input{diagrams/bt/xml}
    \end{minipage}
    \end{tabular}
    \caption{%
        A simple behavior tree to interact verbally using the STT and VLM models optimized in \Cref{fig:pareto} plus TTS capabilities is developed. 
        An initial composition is solved with a free structural variable $\struct$. 
        The conditions of $\struct$ are solved such that the composition is valid. 
        This behavior tree is stored as a new \sw{VLMInteraction} behavior component, which is reasoned over in \Cref{subsec:applications:runtime}. 
        Once a real behavior matching the solved specification for $\struct$ has been implemented, during runtime, the remaining free parametric variables $\param$ are solved and the morphism application of \sw{VLMInteraction} is grounded to the executable behavior tree shown right, translating the symbolic action into real robot behavior.
    }
    \vspace*{-1ex}
    \label{fig:behavior_tree_design}
\end{figure*}
\begin{figure*}
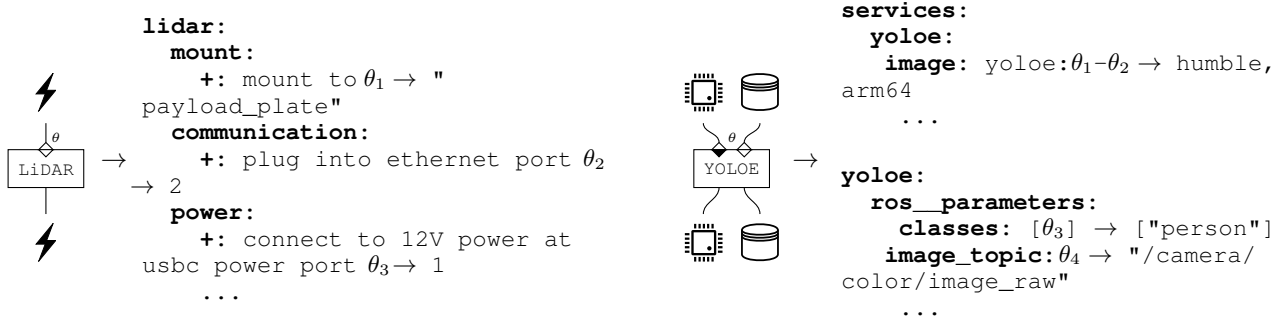

    \centering
    \begin{tabular}{c@{$\;\to$\;\;}c@{\qquad}c@{$\!\to$\;\;\;}c}
    \input{diagrams/lidar/lidar}
    &
    \begin{minipage}{0.35\textwidth}
    \input{diagrams/lidar/params}
    \end{minipage}
    &
    \input{diagrams/yoloe/yoloe}
    &
    \begin{minipage}{0.35\textwidth}
    \input{diagrams/yoloe/compose}
    \input{diagrams/yoloe/params}
    \end{minipage}
    \end{tabular}
    \caption{
        Runtime configurations for a LiDAR hardware component and a YOLOE \cite{yoloe_2025} software component adherent to the standard of \cite{coral_2026}.
        Because the robot in our demonstration (\Cref{fig:demo}) cannot reconfigure its own hardware, morphism applications associated with hardware are grounded to human-readable instructions (left) which must be carried out manually by a human teammate. 
        In our demonstration, these instructions are provided via an augmented reality headset.
        Morphism applications associated with software are grounded into the compose and parameter YAML file formats required by \cite{coral_2026} (right), and all reconfiguration and management of software is performed autonomously.
    }
    \vspace*{-1ex}
    \label{fig:config}
\end{figure*}
\begin{figure*}[htbp!]
    \centering
    \vspace*{1ex}
    \ifsubmission
        \includegraphics[width=0.99\textwidth]{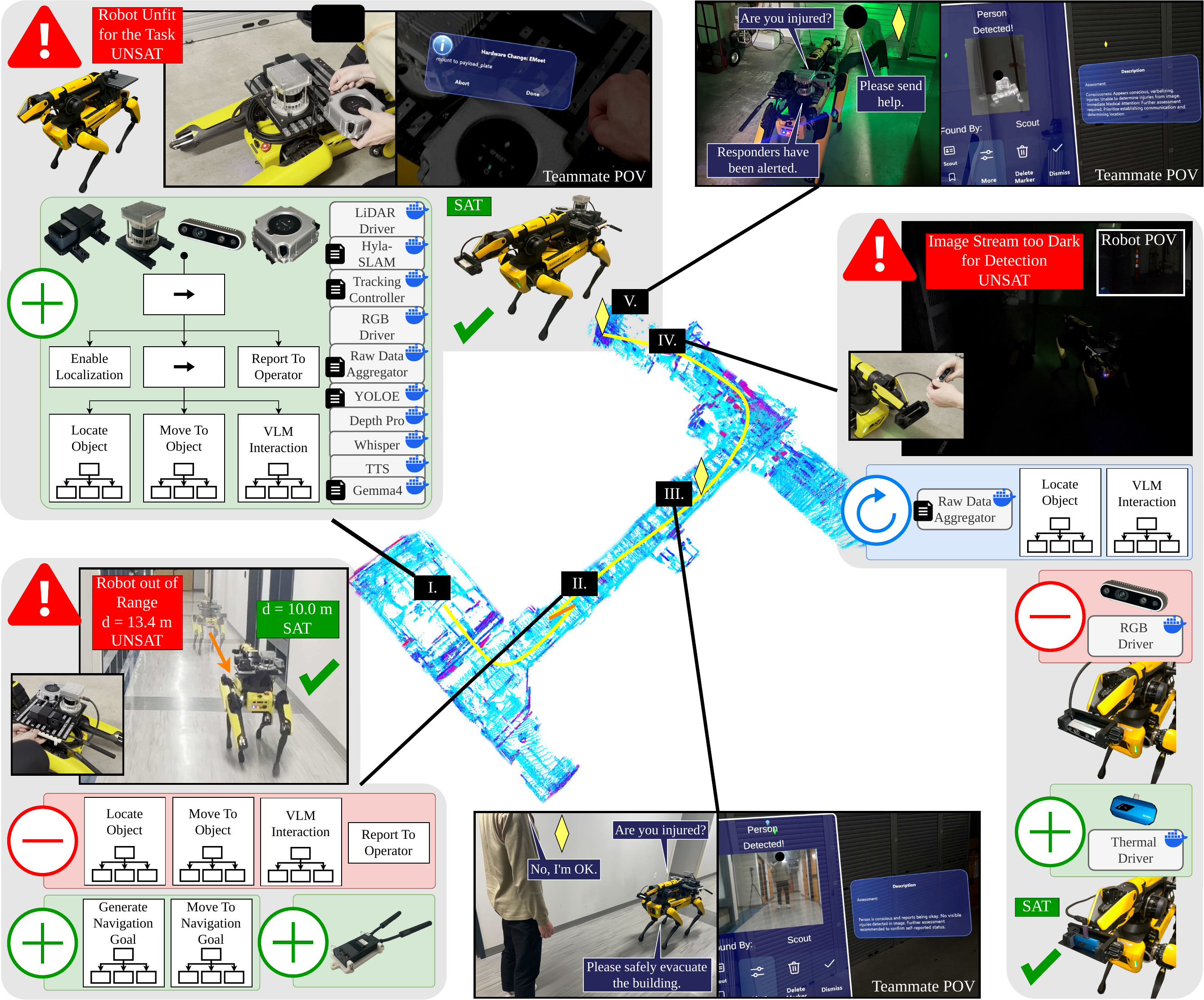}
    \else
        \includegraphics[width=0.99\textwidth]{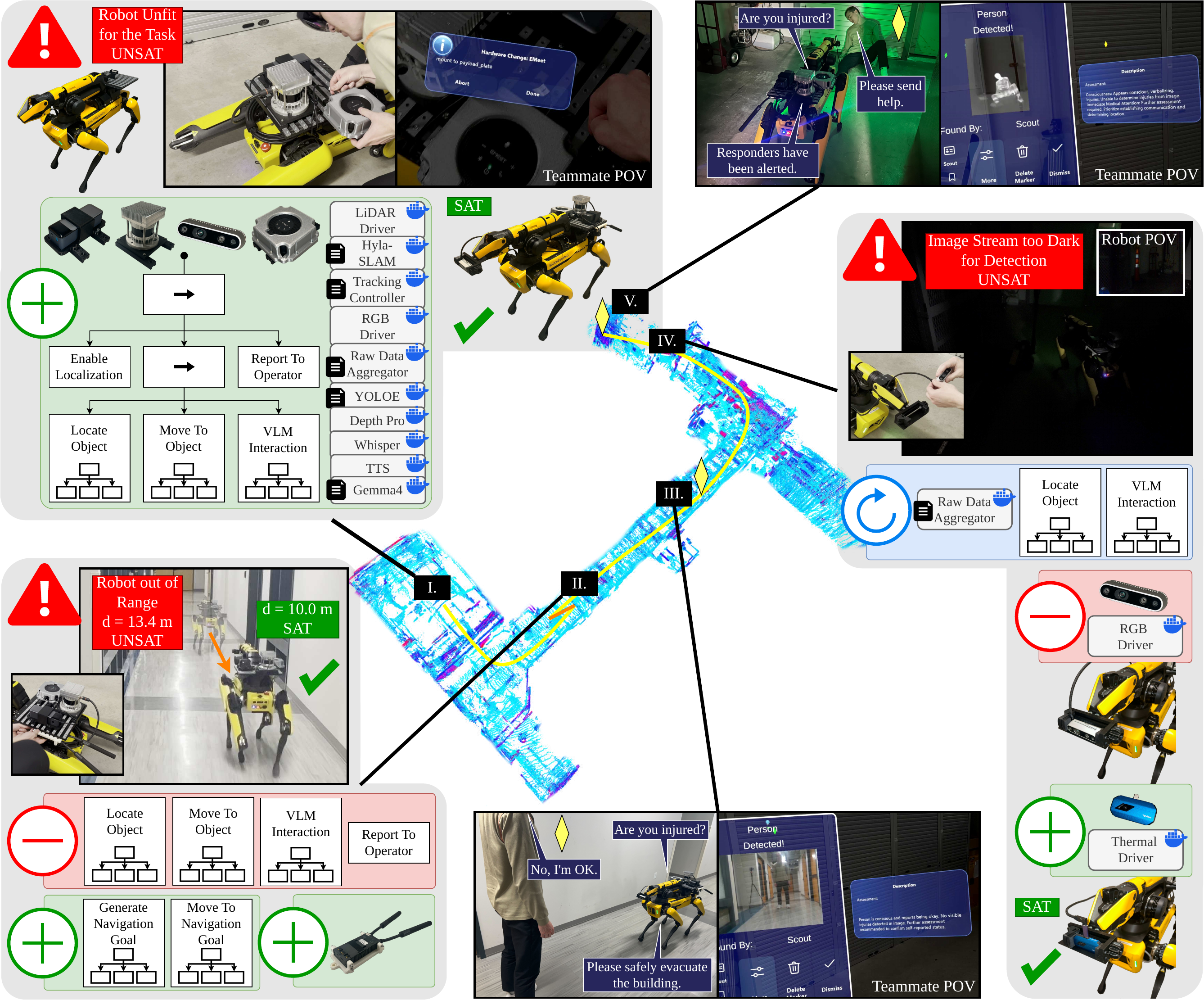}
    \fi
    \caption{%
        Search and Rescue Demonstration: 
        Counter-clockwise from top left: 
        I) A quadruped robot and human teammate with an AR headset are tasked with performing a search-and-rescue mission. 
        The robot cannot perform the task in its base composition ($\unsat$), and solves for a new composition to achieve the task.
        Software and behavior recomposition are autonomous and instructions for hardware recomposition are provided to the teammate via the AR headset.
        Once reintegrated, the robot is $\sat$.
        II) During deployment, the robot exceeds a \SI{10}{\m} range constraint from the teammate.
        The robot tries to solve for a new composition that would render it $\sat$, and generates and executes a new behavior tree to move itself back within communication range.
        To continue exploring further into the facility, the teammate attaches a long-range radio component, which significantly extends the range constraint.
        III) Upon encountering a person, the robot approaches, interacts with the person, and shares their location and an assessment of their condition with the teammate.
        IV) During exploration, the robot enters a dark room and is no longer able to detect people within its RGB image stream, rendering it $\unsat$.
        The robot solves for an updated composition to become $\sat$, which requires a physical swap of the RGB camera for a thermal one, corresponding software changes, and modifications to its behavior tree.
        In the updated composition, the robot is once again $\sat$.
        V) Using the thermal camera, the robot locates and interacts with a person in darkness, providing the same location and assessment information to the remote teammate.
    }
    \vspace*{-1ex}
    \label{fig:demo}
\end{figure*}

Using these measurements and our hybrid quality metric, we created a set of symbolic representations for the five STT model variants and eleven VLM variants and generated the Pareto front shown in \Cref{fig:pareto} (bottom).
Importantly, the Pareto front shows the optimized subcompositions in the context of the resources provided by the deployment computer and all other components deployed in the search-and-rescue mission, not in isolation of the assessment capability alone.
Of the 55 possible subcompositions of STT model and VLM, all 55 are feasible (expected, because the selection of components was guided by the initially solved constraints), and 20 are Pareto-optimal, with the remaining 35 being dominated by at least one other subcomposition.
Using these results, the final STT model and VLM software components were designed, using the Small variant of Whisper and the 26B parameter variant of Gemma4, respectively, as shown in \Cref{fig:pareto} (top right).

Next, we demonstrate solving an integration gap within a partially formed composition using a structural free variable $\struct$.
\Cref{fig:behavior_tree_design} shows a behavior subtree skeleton that uses the STT model and VLM components designed above plus a capability provided by a text-to-speech (TTS) model to interact with the person being assessed.
The \sw{PromptVLM} behavior requires both text and image inputs, but the behavior subtree skeleton does not provide a source for the image input, so the composition is $\unsat$.
A free structural variable $\struct$ is introduced, and its conditions are solved such that the image input is provided to the \sw{PromptVLM} behavior and the composition is $\sat$.
$\struct$ is then implemented as a real behavior \sw{TakePhoto} that takes a ROS image topic as input and provides a snapshot of data from that topic as output.
This interaction composition is stored as a new \sw{VLMInteraction} morphism, and during runtime, its parametric free variables $\param$ are solved and grounded to the subtree shown in \Cref{fig:behavior_tree_design} (right), translating the symbolic action into real robot behavior.

The STT model and VLM subcomposition and interaction subtree designed in this section are deployed in the search-and-rescue mission described in the following subsection.
This application demonstrates that the design-space exploration performed by dedicated co-design tools \cite{co_design_theory_2015,co_design_embodied_intelligence_2021} falls out of our formulation as one query among several rather than requiring separate tooling; because these methods do not release comparable open-source implementations, we do not attempt a runtime comparison, but note that \Cref{fig:pareto} required no purpose-built machinery beyond the model already used for synthesis and runtime recomposition elsewhere in this paper.

\subsection{Runtime Recomposition in a Search-and-Rescue Mission}
\label{subsec:applications:runtime}

In this demonstration, a robot and human teammate are tasked with performing a search-and-rescue mission in an unmapped facility.
The robot is initialized with symbolic representations of a set of hardware components that can be interfaced with it via a simple rail-based mounting mechanism, \cite{coral_2026}-compatible software components already present on the system but not running associated with various capabilities that may be useful during the mission, including drivers for various hardware components, image detection, STT, VLM, TTS, tracking, and SLAM, and behaviors provided by the various software modules that can be composed into behavior trees, also compatible with \cite{coral_2026}.
The human wears an augmented reality (AR) headset that enables in situ communication between the robot and teammate and allows the teammate to see markers in the environment showing the location of the robot and any people detected during the mission.

When provided the search-and-rescue objective, the robot recognizes it is $\unsat$ in its initial (empty) composition, and autonomously generates a new composition for itself to become $\sat$ including interfacing of new hardware modules, starting of various software components, and generating and executing a task-aligned behavior tree.
In addition to the STT model and VLM subcomposition and interaction subtree, other components utilized in the composition include a LiDAR sensor that enables a software component based on \cite{hyla_slam_2025} required to allow the robot to localize itself in the unfamiliar environment and a software component based on YOLOE \cite{yoloe_2025} polled at high frequency to detect people in the environment before initiating the interaction process.
As examples of the grounding of these symbolic components into actionable updates to the robot state, we show the free parametric variables $\param$ solved for the LiDAR and YOLOE components in \Cref{fig:config}.
In this demonstration, our robot is not designed to have self-reconfigurable hardware, so hardware components are grounded to a set of human-readable instructions related to the physical mounting and establishment of power and communication connections with the robot, as shown in \Cref{fig:config} (left).
These instructions are provided to the human teammate via the AR headset, who then must perform the physical reconfiguration following the instructions.
For software components, grounding includes generating the compose and parameter YAML files required by \cite{coral_2026} with the solved parametric variables $\param$, as shown in \Cref{fig:config} (right).
Software recomposition is managed fully autonomously by the robot, including starting, stopping, and reparameterizing software components as required by the generated composition.
Finally, grounding of behavior trees includes generating the XML file required by \cite{coral_2026} with the solved parametric variables $\param$, as shown in \Cref{fig:behavior_tree_design} (right).
Behavior tree recomposition is also managed autonomously.
The complete robot composition including hardware, software, and behavior is shown in \Cref{fig:demo}-I.

The robot periodically checks whether it is still $\sat$ in its current composition during deployment.
In \Cref{fig:demo}-II, it recognizes it has exceeded a communication-related range constraint from the human teammate, and solves for a new composition that includes a behavior tree to walk itself back into range to become $\sat$ again.
To traverse further into the facility, the teammate attaches a long-range radio, which significantly extends this communication range.
In \Cref{fig:demo}-III, the robot detects, moves to, interacts with, and provides an assessment of a detected person within the environment, sharing the assessment and their location with the teammate via the AR headset.
While traversing deeper into the facility, the robot enters a dark room and its image stream becomes too dark to detect people (implemented as a health check on the camera driver computing median pixel brightness), as shown in \Cref{fig:demo}-IV.
The robot recognizes it is $\unsat$, and regenerates a new composition to become $\sat$ again, which includes swapping its RGB camera with a thermal camera, stopping the RGB camera driver software and starting the thermal camera driver software, and reinitializing another software component and behaviors in its behavior tree with new parameters.
In this composition, the robot is once again $\sat$ and is able to detect, move to, interact with, and assess another person detected in the dark room, as shown in \Cref{fig:demo}-V.

Throughout this demonstration, the robot autonomously determines when it has become $\unsat$ due to changes in its task (\Cref{fig:demo}-I), state (\Cref{fig:demo}-II), or environment (\Cref{fig:demo}-IV) and generates new compositions to become $\sat$ again, including grounding symbolic components into actionable updates to its hardware, software, and behavior.
These reconfiguration events are used as benchmarks to evaluate the performance of our approach in \Cref{sec:evaluation}, and single-edit perturbations of the same deployment scenario are used in the diagnosis, relaxation, and recomposition studies in \Cref{subsec:eval:queries}.

\section{Evaluation}
\label{sec:evaluation}

Our primary claim is that lifelong recomposition is a family of queries over one persistent model, of which planning a single composition is only one.
Our evaluation is organized accordingly, as follows.
We first establish that our solver is competitive against baselines at the task of planning a single feasible or optimal composition so that the remaining comparisons are not excused by planning weakness (\Cref{subsec:eval:planning}).
We then measure two further properties supporting lifelong recomposition, including:
\begin{itemize}
    \item How much model modification a deployment-time change requires before a formalism can re-solve (\Cref{subsec:eval:perturbation}); and
    \item Whether the formalism can answer questions that arise when re-solving fails and endangers the deployed system (\Cref{subsec:approach:queries}), including why the composition is unsatisfiable, what the minimal changes are to make it satisfiable, and how much the deployed system must be recomposed (\Cref{subsec:eval:queries}).
\end{itemize}
Finally, \Cref{tab:capabilities} summarizes the expressiveness of our approach and the baselines in terms of the requirements on the modeler and the variable types that can be reasoned over and co-optimized. 

\subsection{Setup and Baselines}
\label{subsec:eval:setup}
We compare against official implementations of the three planning formalisms nearest our setting, each in its strongest configuration: (1) PDDL 2.1 \cite{pddl2_1_2003} with ENHSP \cite{enhsp_2016} as an optimal numeric backend planner; (2) PDDLStream \cite{pddlstream_2020}, both in incremental sampler (PS$_{\text{s}}$) and adaptive optimizer (PS$_{\text{o}}$) configurations with Fast Downward \cite{fast_downward_2006} as the backend planner; and (3) SMTPlan+ \cite{smtplanplus_2016}, a satisficing PDDL+ planner that uses the Z3 SMT solver \cite{z3_2008} as its backend.
Our implementation similarly uses Z3 as its backend solver, so we can attribute differences in performance to the formalism rather than the solver.

Every measurement in this section for every approach was taken onboard the same robot's NVIDIA Jetson AGX Orin from \Cref{sec:applications}, so reported times capture the expected performance on an edge device in a real deployment.
All reported times are the median over 10 trials, and interquartile ranges are below \SI{0.6}{\second} for all findings except one (\SI{1.26}{\second}, ours on \Cref{fig:demo}-II).

We establish three rules for the comparison.
First, every baseline receives its best-faith encoding, including encodings that strengthen the baseline beyond its standard usage.
For example, for PDDL 2.1 with ENHSP, we authored an additional repeatable-increment encoding that absorbs some numeric changes that its standard encoding cannot.
Second, no baseline is ever fed an ill-formed input.
Negative results are correct outputs of correct encodings, not failures of the encoding.
Third, mission-time information, including sensed or computed values, operator requests, and goals, may be written into \emph{problem} files for every formalism.
What distinguishes formalisms is what must change in the \emph{model} (domain, streams, optimizers, costs, etc.) when the system updates.
We release these encodings and our full results with our code\textsuperscript{\ref{fn:code}}.

\begin{table}[t]
\centering
\vspace{1ex}
\caption{%
    Solve times (seconds) on the search-and-rescue deployment (\Cref{fig:demo}-I/II/IV) and three microbenchmarks, all measured onboard the robot's NVIDIA Jetson AGX Orin (median of 10 trials). 
    $\sat$ rows seek any feasible plan; OPT rows seek an optimum.
    PDDL with ENHSP \cite{enhsp_2016};
    PS$_{\text{s}}$/PS$_{\text{o}}$: PDDLStream sampler/optimizer~\cite{pddlstream_2020} with Fast Downward (FD) \cite{fast_downward_2006}; 
    SMTP: SMTPlan+ \cite{smtplanplus_2016}. 
    \cmark~=~optimal ($\opt$) or feasible ($\sat$) plan found; 
    \xmark~=~suboptimal ($\subopt$) or infeasible ($\unsat$); 
    \notapp~=~optimizer inapplicable (feasibility-only row, no numeric objective, so the adaptive config reduces to the sampler);
    \notenc~=~a valid encoding exists in principle but we did not author it for this problem: PDDLStream on the full-scale search-and-rescue model, whose hand-written stream/optimizer bodies are a substantial modeling effort we instead demonstrate on the microbenchmarks that isolate the same capabilities.
}
\vspace{1ex}
\scriptsize
\setlength{\tabcolsep}{4pt}
\renewcommand{\arraystretch}{1.15}
\begin{tabular*}{\columnwidth}{@{\extracolsep{\fill}}l c c c c c@{}}
\toprule
Scenario
    & \makecell[c]{PDDL\\(ENHSP)} 
    & \makecell[c]{PS$_{\text{s}}$\\(FD)} 
    & \makecell[c]{PS$_{\text{o}}$\\(FD)} 
    & SMTP 
    & Ours \\
\midrule
\multicolumn{6}{@{}l}{\textit{Search and rescue demonstration (\Cref{fig:demo}):}}\\
\Cref{fig:demo}-I ($\sat$)
    & \makecell[c]{\cmark\,(1.01\,s)\\\satmark}
    & \notenc 
    & \notenc 
    & \makecell[c]{\cmark\,(2.63\,s)\\\satmark}
    & \makecell[c]{\cmark\,(6.46\,s)\\\satmark} \\
\addlinespace
\Cref{fig:demo}-II ($\opt$)
    & \makecell[c]{\xmark\,(0.58\,s)\\\suboptmark}
    & \notenc 
    & \notenc 
    & \makecell[c]{\xmark\,(0.21\,s)\\\suboptmark}
    & \makecell[c]{\cmark\,(4.85\,s)\\\optmark} \\
\addlinespace
\Cref{fig:demo}-IV ($\sat$)
    & \makecell[c]{\cmark\,(1.02\,s)\\\satmark}
    & \notenc 
    & \notenc 
    & \makecell[c]{\cmark\,(15.9\,s)\\\satmark}
    & \makecell[c]{\cmark\,(6.66\,s)\\\satmark} \\
\midrule
\multicolumn{6}{@{}l}{\textit{Microbenchmark:}}\\
GPU alloc.\ ($\opt$)
    & \makecell[c]{\cmark\,(0.51\,s)\\\optmark}
    & \makecell[c]{\cmark\,(13.1\,s)\\\optmark}
    & \makecell[c]{\cmark\,(0.56\,s)\\\optmark}
    & \makecell[c]{\xmark\,(0.05\,s)\\\suboptmark}
    & \makecell[c]{\cmark\,(0.12\,s)\\\optmark} \\
\addlinespace
Navigation ($\opt$)
    & \makecell[c]{\xmark\,(0.48\,s)\\\suboptmark}
    & \makecell[c]{\xmark\,(5.05\,s)\\\suboptmark}
    & \makecell[c]{\cmark\,(0.53\,s)\\\optmark}
    & \makecell[c]{\xmark\,(0.02\,s)\\\suboptmark}
    & \makecell[c]{\cmark\,(0.02\,s)\\\optmark} \\
\addlinespace
Detection ($\sat$)
    & \makecell[c]{\cmark\,(0.48\,s)\\\satmark}
    & \makecell[c]{\cmark\,(0.03\,s)\\\satmark}
    & \notapp 
    & \makecell[c]{\cmark\,(0.05\,s)\\\satmark}
    & \makecell[c]{\cmark\,(0.11\,s)\\\satmark} \\
\bottomrule
\end{tabular*}
\label{tab:runtimes}
\end{table}

\begin{table*}[t]
\centering
\vspace{1ex}
\caption{%
    Absorbing mission-time change. 
    Each row perturbs a solved problem by one runtime event; cells report the \emph{model} edits (lines outside the problem file, comments excluded) a human must make before re-solving, and the re-solve outcome with median time over 10 trials on the robot's Jetson.
    Model-line counts are \emph{per configuration}: added$+$removed substantive lines in the \texttt{.pddl} files each config loads--sampler domain$+$stream, optimizer domain$+$optimizer--excluding problem definitions. 
    Outcome markers match \Cref{tab:runtimes}: 
    \cmark~=~optimal/feasible; 
    \xmark~=~suboptimal/infeasible; 
    \notapp~=~optimizer inapplicable.
}
\vspace{1ex}
\scriptsize
\setlength{\tabcolsep}{4pt}
\renewcommand{\arraystretch}{1.15}
\begin{tabular*}{\textwidth}{@{\extracolsep{\fill}}l c c c c c@{}}
\toprule
Mission-time event 
    & \makecell[c]{PDDL\\(ENHSP)} 
    & \makecell[c]{PS$_{\text{s}}$\\(FD)} 
    & \makecell[c]{PS$_{\text{o}}$\\(FD)} 
    & SMTP 
    & Ours \\
\midrule
\makecell[l]{GPU upgrade:\\GPU pool $+$\SI{1}{\giga\byte} ($\opt$)}
  & \makecell[c]{\numpad{0}\,L $\to$ \xmark$^\dagger$\!(0.51\,s)\\\suboptmark}
  & \makecell[c]{\numpad{0}\,L $\to$ \cmark\,(15.3\,s)\\\optmark}
  & \makecell[c]{\numpad{0}\,L $\to$ \cmark\,(0.55\,s)\\\optmark}
  & \makecell[c]{\numpad{0}\,L $\to$ \xmark\,(0.05\,s)\\\suboptmark}
  & \makecell[c]{\numpad{0}\,L $\to$ \cmark\,(0.12\,s)\\\optmark} \\
\addlinespace
\makecell[l]{Re-tasking:\\objective changed ($\opt$)}
  & \makecell[c]{\numpad{34}\,L $\to$ \cmark\,(0.51\,s)\\\optmark}
  & \makecell[c]{\numpad{30}\,L $\to$ \cmark\,(13.1\,s)\\\optmark}
  & \makecell[c]{\numpad{32}\,L $\to$ \cmark\,(0.55\,s)\\\optmark}
  & \makecell[c]{\numpad{34}\,L $\to$ \xmark\,(0.06\,s)\\\suboptmark}
  & \makecell[c]{\numpad{0}\,L $\to$ \cmark\,(0.12\,s)\\\optmark} \\
\addlinespace
\makecell[l]{Increased precision:\\exact real measurement (7.73) ($\sat$)}
  & \makecell[c]{\numpad{0}\,L $\to$ \xmark\,(0.46\,s)\\\unsatmark}
  & \makecell[c]{\numpad{24}\,L $\to$ \xmark\,(3.90\,s)\\\unsatmark}
  & \makecell[c]{\numpad{22}\,L $\to$ \cmark\,(0.53\,s)\\\satmark}
  & \makecell[c]{\numpad{0}\,L $\to$ \cmark\,(0.02\,s)\\\satmark}
  & \makecell[c]{\numpad{0}\,L $\to$ \cmark\,(0.02\,s)\\\satmark} \\
\addlinespace
\makecell[l]{Unmodeled request:\\out-of-vocabulary class ($\sat$)}
  & \makecell[c]{\numpad{0}\,L $\to$ \xmark\,(0.45\,s)\\\unsatmark}
  & \makecell[c]{\numpad{0}\,L $\to$ \xmark\,(3.62\,s)\\\unsatmark}
  & \notapp
  & \makecell[c]{\numpad{0}\,L $\to$ \xmark\,(1.68\,s)\\\unsatmark}
  & \makecell[c]{\numpad{0}\,L $\to$ \cmark\,(0.11\,s)\\\satmark} \\
\bottomrule
\end{tabular*}
\\[3pt]
{\scriptsize 
    $^\dagger$The enumerated encoding re-solves but returns the \emph{stale} optimum (the new one is off its grid); a repeatable-increment re-encoding we authored for ENHSP recovers the optimum at the cost of 36 domain lines.
}
\label{tab:perturbations}
\end{table*}

\begin{table}[t]
\centering
\caption{%
    Lifelong queries on the deployed search-and-rescue model (\Cref{fig:demo}-I).
    Top rows: a budget cut renders the deployment infeasible.
    Bottom row: a cheaper interchangeable camera enters the catalog while the incumbent remains deployed.
    Baselines are given their best-faith workaround where one exists.
    Churn = components added$+$removed relative to the deployed system.
    PDDLStream is omitted: it has no encoding of the full-scale search-and-rescue model (\Cref{tab:runtimes}, \notenc), so no workaround column can be measured for it.
}
\vspace{1ex}
\scriptsize
\setlength{\tabcolsep}{2.5pt}
\renewcommand{\arraystretch}{1.15}
\begin{tabular*}{\columnwidth}{@{\extracolsep{\fill}}l l l l@{}}
\toprule
Query
    & \multicolumn{1}{c}{\makecell[c]{PDDL\\(ENHSP)}}
    & \multicolumn{1}{c}{SMTP}
    & \multicolumn{1}{c}{Ours} \\
\midrule
\makecell[l]{Diagnose: why\\does the deployed\\config.\ fail?}
      & \makecell[l]{``unsolvable''\\no cause\\(0.53\,s)}
      & \makecell[l]{timeout$^\dagger$\\no certificate}
      & \makecell[l]{minimal unsat core:\\1 of 94 constraints\\(high LiDAR cost)\\(0.08\,s)} \\
\addlinespace
\makecell[l]{Relax: minimal\\restoration?}
      & \makecell[l]{$+$19{,}860 USD after\\22 full runs with\\bisection replan (19.6\,s);\\parameter must\\be designated}
      & \makecell[l]{impractical:\\every $\unsat$\\probe times out}
      & \makecell[l]{$+$19{,}860 USD\\exact, 1 call (11.4\,s);\\all resources relaxed\\jointly, no per-\\parameter loop} \\
\addlinespace
\makecell[l]{Recompose:\\minimal change?}
      & \makecell[l]{churn 4 (1.22\,s)\\no stability term\\in the encoding}
      & \makecell[l]{churn 4 (8.97\,s)\\no stability term\\in the encoding}
      & \makecell[l]{churn 0 (5.35\,s)\\components held\\by \Cref{eq:reconfigure}$^\ddagger$} \\
\bottomrule
\end{tabular*}
\\[3pt]
{
    \scriptsize
    $^\dagger$Iterative-deepening satisficing search cannot certify infeasibility; exceeded a 300\,s budget on every trial.
    \quad
    $^\ddagger$\emph{objective-first} lexicographic priority recovers the cost-optimal churn-4 solution; \emph{stability-first} preserves the incumbent.
}
\label{tab:queries}
\end{table}

\begin{table}[t]
\centering
\caption{%
    What each planner supports and what the modeler must supply to reach the optimum on the coupled-parameter benchmarks, grouped as modeling burden, expressible free-variable classes, and outcome. 
    \cmark~=~yes/required, \xmark~=~no/not needed; \textit{enum.}/\textit{disc.}~=~ expressible only via pre-enumeration or discretization; $\varnothing$~=~none.
}
\vspace{1ex}
\scriptsize
\setlength{\tabcolsep}{3.5pt}
\renewcommand{\arraystretch}{1.0}
\begin{tabular*}{\columnwidth}{@{\extracolsep{\fill}}l c c c c c@{}}
\toprule
    & PDDL 
    & PS$_{\text{s}}$ 
    & PS$_{\text{o}}$ 
    & SMTP 
    & Ours \\
\midrule
\multicolumn{6}{@{}l}{\textit{Modeler must hand-encode:}}\\
Enumerated/discretized values
    & \cmark 
    & \cmark 
    & \xmark
    & \xmark 
    & \xmark \\
Per-action costs tuned to objective
    & \cmark 
    & \cmark 
    & \cmark 
    & \xmark\rlap{$^\ddagger$}
    & \xmark \\
Solution structure in model
    & \xmark
    & \xmark
    & \cmark 
    & \xmark 
    & \xmark \\
\midrule
\multicolumn{6}{@{}l}{\textit{Supported variable types:}}\\
Activations $\activation$ 
    & \cmark 
    & \cmark 
    & \cmark 
    & \cmark 
    & \cmark \\
Parameters $\param$ ($\mathbb{R},\mathbb{Z}$) 
    & \textit{enum.} 
    & \textit{disc.} 
    & \cmark 
    & \cmark 
    & \cmark \\
Parameters $\param$ ($\mathrm{String}, \mathrm{Set}, \dots$) 
    & \xmark 
    & \xmark 
    & \xmark 
    & \xmark 
    & \cmark \\
Structures $\struct$ 
    & \xmark 
    & \xmark 
    & \xmark 
    & \xmark 
    & \cmark \\
\midrule
\multicolumn{6}{@{}l}{\textit{Outcome:}}\\
Reaches optimum
    & \cmark\rlap{$^\dagger$}
    & \cmark\rlap{$^\dagger$}
    & \cmark 
    & \xmark\rlap{$^\ddagger$}
    & \cmark \\
Optimum \emph{without} prior knowledge
    & \xmark 
    & \xmark 
    & \xmark 
    & \xmark 
    & \cmark \\
Optimized in single solve 
    & $\activation$ 
    & $\activation$ 
    & $\activation \otimes \param$ 
    & $\varnothing$\rlap{$^\ddagger$}
    & $\activation \otimes \param \otimes \struct$ \\ 
\bottomrule
\end{tabular*}
\\[3pt]
{\scriptsize 
    $^\dagger$Only when the optimum coincides with an enumerated/discrete
    value. 
    \quad 
    $^\ddagger$SMTPlan+ is a \emph{satisficing} PDDL+ planner: its SMT backend binds continuous parameters exactly (hence \cmark on $\param\in\mathbb{R},\mathbb{Z}$, and no discretization is required), but it ignores optimization metrics; it is shown here for expressiveness, with runtimes in \Cref{tab:runtimes,tab:perturbations}.
}
\label{tab:capabilities}
\end{table}

\subsection{Planning Performance}
\label{subsec:eval:planning}
\Cref{tab:runtimes} reports solve times on the three recomposition events of the search-and-rescue deployment (\Cref{fig:demo}-I/II/IV), which are encoded verbatim in each formalism, and three smaller microbenchmarks that isolate individual capabilities.
The first microbenchmark, GPU allocation, considers allocation of a set of GPU-intensive processes sharing a single pool of resources with an objective to maximize the confidence of one process which is coupled to its resource consumption from the shared pool. 
The second, Navigation, represents an objective to travel as far as possible given a coupled battery constraint.
The third, Detection, considers creating a detection pipeline that can detect a provided class from an input datastream of image data.
We claim no planning speed advantage for our approach: ENHSP-backed PDDL 2.1 is faster on the full-scale problems and SMTPlan+ is often the fastest of all on the smaller problems.
Rather, our advantage is in the expressivity our approach affords.
For example, once an objective requires a free real, an open-world symbolic value, or several parameters coupled by a global objective, PDDL struggles to express it (returning $\subopt$ grid values on the \Cref{fig:demo}-II and Navigation microbenchmark rows) and PDDLStream reaches the optimum only where hand-authored streams already encode the solution structure.
SMTPlan+ shows that the separation is due to formulation rather than solving approach, as despite sharing the same backend solver, because it is a satisficing planner, it returns $\subopt$ solutions on every optimization row.
Our solver synthesizes all free variables from each morphism's conditions and co-optimizes them in a single query, reaching the optimum on every row.

This added expressivity of our approach remains compatible with real-time deployment, as the full search-and-rescue problems solve in \SIrange{4.9}{6.7}{\second} onboard the edge device, which is well within the budget for an event-driven recomposition that is only triggered  when the robot becomes $\unsat$, not at every control step.


\subsection{Absorbing Change During Deployment}
\label{subsec:eval:perturbation}
In the long-lived recompositional robotics context, the ability to re-solve under changes is just as important as initial one-shot configuration generation.
\Cref{tab:perturbations} measures four representative mission-time events, each applied as a one-edit perturbation to previous benchmark problems.

The first modification, GPU upgrade, is an update to the GPU allocation microbenchmark that increases the available GPU memory by \SI{1}{\giga\byte}.
The second modification, Re-tasking, is another update to the GPU allocation microbenchmark that changes the objective from maximizing the confidence of a single process to minimizing the total GPU footprint of all processes. 
The third modification, Sensed value, is an update to the Navigation microbenchmark that increases the precision of the sensed battery level from a single decimal place to two decimal places.
The fourth modification, Operator request, is an update to the Detection microbenchmark that adds an arbitrary detection class provided by an operator at runtime rather than coming from a pre-enumerated vocabulary.
For each formalism, we report the \emph{model} edits a human must make before the planner can even be re-run, together with the re-solve outcome and time.

Several observations are apparent from \Cref{tab:perturbations}.
For the first modification, the PDDL encoding solved with ENHSP re-solves the problem without any model edits, but returns a stale optimum because the new optimum lies off its pre-enumerated grid.
Recovering the true optimum requires the repeatable-increment re-encoding (36 domain lines) we authored that increases in \SI{1}{\giga\byte} increments.
Such a repeatable-increment re-encoding is also sensitive to the required precision, as changing the GPU upgrade to \SI{500}{\mega\byte} would require another re-encoding that increases in increments $\leq$ \SI{500}{\mega\byte}.
Both PDDLStream configurations are $\opt$ without any changes, but SMTPlan+ returns a $\subopt$ solution because it is a satisficing planner, while our approach recovers $\opt$ without any model edits.
For the second modification, all baseline encodings require model edits to redefine the costs associated with actions such that they are consistent with the new objective.
SMTPlan+ also still returns $\subopt$, while our approach again recovers $\opt$ without any model edits.
For the third modification, PDDL and PDDLStream approaches require model edits, are $\unsat$, or both, while SMTPlan+ recovers the exact real value without any model edits, inheriting this capability from its underlying SMT solver, as does ours.
For the fourth modification, all enumeration-based encodings return $\unsat$ because pre-enumeration cannot anticipate an open vocabulary and no finite model edit would help.
In comparison, our approach can recover a $\sat$ result due to \Cref{prop:grounding}.
Across all four modifications, our approach requires zero model edits to recover $\sat$ and $\opt$ results while most of the baselines require 22--34 model edits, produce $\unsat$ or $\subopt$ results, or both.
In the context of a long-lived, recompositional system, this difference is significant, as the baselines require a human to intervene and re-encode the model while our approach can recover without changes. 

\subsection{Queries Beyond Plan Existence}
\label{subsec:eval:queries}
\Cref{tab:queries} evaluates the three queries of \Cref{subsec:approach:queries} on a variant of the base search-and-rescue problem from \Cref{fig:demo}-I in which financial costs for procuring hardware modules are introduced.

In the first case, a modest budget of 500 USD is provided, which is sufficient to procure the inexpensive hardware modules including the RGB camera, speaker/microphone, and external battery but insufficient to procure the more expensive LiDAR sensor (20,000 USD) utilized in \Cref{fig:demo}-I.
When queried for the cause of unsatisfiability, our solver identifies the LiDAR's financial cost as the sole member of $\mus$, and reports that the problem is $\unsat$ in \SI{0.08}{\second}, while ENHSP-backed PDDL reports the problem is unsolvable with no cause.
SMTPlan+ cannot certify infeasibility or provide an explanation.

When queried for the minimal restoration, our solver provides the solution of increasing the available budget by 19,860~USD in \SI{11.4}{\second} without being told which parameter to relax.
The strongest workaround we could construct for a black-box planner is a bisection loop re-running ENHSP at candidate budgets. 
It recovers the same solution after 22 full iterations in \SI{19.6}{\second}, though it must be provided the parameter(s) to relax, must assume feasibility is monotonic in the parameter(s), and must be repeated per candidate parameter.
For SMTPlan+, this same workaround is impractical outright because each $\unsat$ probe times out.

The recompose query uses a slightly different problem in which the budget is increased substantially such that the LiDAR's cost is not prohibitive, but a new RGB camera module is introduced that is less expensive than the existing and deployed RGB camera and the objective is to minimize the total financial cost of the composition.
Both ENHSP-backed PDDL and SMTPlan+ return the cost-optimal solution to the revised problem, which has a total churn of 4 to exchange the existing RGB camera and drivers.
Our approach can similarly recover the cost-optimal solution, but because the cost-versus-stability tradeoff is explicit from \Cref{eq:reconfigure}, it can instead return a solution that preserves the existing deployed configuration with a total churn of 0 at the expense of a slightly higher total cost.

\subsection{Expressivity Summary}
\label{subsec:eval:summary}
\Cref{tab:capabilities} concisely captures the differences in what must be hand-encoded, which free-variable classes can be expressed, and what types of variables can be co-optimized in each benchmark and our formalism.
In summary, the baselines exhibit excellent performance for finding optimal plans over fixed models, which is the task for which they are designed.
However, each new capability we consider important to the objective of lifelong recomposition requires hand-authored workarounds, requires a re-encoding, or is outright inexpressible.
In comparison, our approach enables expression of the same capabilities as minimality queries over the same persistent, categorical model and supports simultaneous co-optimization of all free variables in a single query. 

\section{Discussion and Limitations}
\label{sec:limitations}

Practical deployment of our approach is subject to several considerations.
First, the efficacy of reasoning depends on the precision of the underlying ontology.
For example, while the RGB and thermal image streams in \Cref{subsec:applications:runtime} produce the same type of data, they possess distinct semantic properties.
In our search and rescue task, we treat them equivalently for person detection, but they would not work interchangeably to detect objects with lower thermal contrast.
As with many other formal methods-based approaches, such semantic differences must be explicitly captured.

Second, structural free variables $\struct$ represent the weakest condition assignments that render the composition feasible, not realizable designs.
Nothing compels the solver to invent or combine realistic conditions, and because structural free variables are symbolic morphisms rather than fully-formed hardware, software, or behavior components, implementation of a practical component that matches the synthesized $\struct$ must also be performed externally (as in \Cref{subsec:applications:design}).
It would also be possible to match the synthesized $\struct$ against a catalog of existing components, but this is not a capability we have implemented.

Third, as described in \Cref{sec:applications}, while the robot in our search-and-rescue application has full agency over its own behavior and software, it does not have the ability to reconfigure its own hardware.
Rather than autonomously interfacing new hardware with itself, it generates reconfiguration instructions that must be enacted by a human teammate aided by an AR headset.
This is not a limitation with our approach, but rather a realistic limitation of the overwhelming majority of real robotic systems and even many modular and reconfigurable systems which are not designed to support self-reconfigurable hardware \cite{plug_and_produce_review_2024,concert_2026}.
If our approach were deployed on a system that could reconfigure its own hardware, the modeling/solving approach would not change at all; instead, the \emph{grounding} of the symbolic representations to reality would be adjusted.
For example, rather than generating instructions for a person to follow and sharing them via an AR headset, the symbolic representations could be grounded to self-reconfiguration actions that the robot could perform autonomously.

Fourth, the STT and VLM subcomposition developed in \Cref{subsec:applications:design} inherits the assumptions of its quality metric.
That is, stage successes are treated as independent and the per-variant constants are point estimates drawn from published benchmarks and on-device measurements rather than distributions over operating conditions.
The enumerated front is exact for the model and is bounded by the fidelity of the constants to the deployed system.

Finally, our evaluation compares against planning formalisms which were built to synthesize plans over a fixed model, so evaluating them on model evolution, infeasibility explanation, and stability exercises them outside the envelope for which they were designed.
We compare against them nonetheless because they are the only nearby systems with mature, open implementations that we could run on the same problems and hardware.
The families that address the remaining parts of our setting are even less comparable, as co-design frameworks rarely release tooling that can be utilized on an external problem, and self-adaptive architectures presuppose a fixed adaptation space rather than synthesizing one.
Nor is there an established benchmark for component-based lifelong recomposition spanning hardware, software, and behavior, so the studies in \Cref{sec:evaluation} required us to author both the one-edit problem variants and the best-faith baseline encodings evaluated on them.
These efforts do not substitute a neutral benchmark, but they provide a practical means to evaluate our approach and facilitate comparison with existing methods.

\section{Conclusions}
\label{sec:conclusions}

By formalizing robotic systems as dynamic, categorical circuits, we demonstrate how the boundary between design and deploy stages in the traditional design-then-deploy paradigm can be blurred.
Our approach uses the same categorical model to solve for Pareto-optimal compositions before deployment and to answer questions about the deployed system's feasibility, restoration costs, and minimal changes required for continued operation as the system's capabilities, tasks, or environment unexpectedly change.
We demonstrate our approach in a search-and-rescue application, from pre-deployment design of individual components and subcompositions to runtime recomposition spanning hardware, software, and behavior.
We also evaluate our approach against optimal numeric, stream-based, and SMT-based planners, showing that our framework can answer questions about the deployed system that are either inexpressible or require hand-authored workarounds in the baselines.
Ultimately, this framework provides a mathematical and computational foundation for resilient robotic systems capable of lifelong recomposition in response to changes within themselves, their tasks, and their environments.

{\footnotesize
\bibliographystyle{IEEEtran}
\bibliography{bibliography}
}

\newpage
\vfill

\end{document}